\documentclass{article}
\usepackage{geometry}
\usepackage[pdftex]{graphicx}
\usepackage{subcaption}
\usepackage{cite}
\usepackage{array}
\usepackage{amsmath}
\usepackage{amssymb}
\usepackage{mathtools}
\usepackage{amsthm,bm}
\usepackage{booktabs}
\usepackage{multirow}
\usepackage{makecell}
\usepackage{enumerate}
\usepackage{color}
\usepackage{algorithmic}
\usepackage{algorithm}
\usepackage{url}
\usepackage[sectionbib]{chapterbib}

\newtheorem{theorem}{Theorem}
\newtheorem{lemma}{Lemma}

\newtheorem{assumption}{Assumption}
\newtheorem{definition}{Definition}

\begin{document}

\title{Learnability in Online Kernel Selection with Memory Constraint via Data-dependent Regret Analysis
\thanks{This is a pre-print of an article published in Journal of Computer Science and Technology. The final authenticated version is available at: https://doi.org/10.1007/s11390-024-2896-z}}

%\author{\IEEEauthorblockN{Ole-Christoffer Granmo}
%\IEEEauthorblockA{Department of ICT\\
%University of Agder\\
%Norway\\
%Email: ole.granmo@uia.no}}

\date{}

\author{
Junfan Li, Shizhong Liao$\dag$\\
College of Intelligence and Computing, Tianjin University, Tianjin 300350, China\\
\{junfli,szliao\}@tju.edu.cn\\
$\dag$~Corresponding Author}

\maketitle

\begin{abstract}
  Online kernel selection is a fundamental problem of online kernel methods.
  In this paper,
  we study online kernel selection with memory constraint
  in which the memory of kernel selection and online prediction procedures
  is limited to a fixed budget.
  An essential question is
  what is the intrinsic relationship among online learnability, memory constraint, and data complexity?
  To answer the question,
  it is necessary to show the trade-offs between regret and memory constraint.
  Previous work gives a worst-case lower bound depending on the data size,
  and shows learning is impossible within a small memory constraint.
  In contrast, we present distinct results by offering data-dependent upper bounds
  that rely on two data complexities:
  kernel alignment and the cumulative losses of competitive hypothesis.
  We propose an algorithmic framework
  giving data-dependent upper bounds for two types of loss functions.
  For the hinge loss function,
  our algorithm achieves an expected upper bound depending on kernel alignment.
  For smooth loss functions,
  our algorithm achieves a high-probability upper bound
  depending on the cumulative losses of competitive hypothesis.
  We also prove a matching lower bound for smooth loss functions.
  Our results show that if the two data complexities are sub-linear,
  then learning is possible within a small memory constraint.
  Our algorithmic framework depends on a new buffer maintaining framework
  and a reduction from online kernel selection to prediction with expert advice.
  Finally,
  we empirically verify the prediction performance of our algorithms on benchmark datasets.
\end{abstract}


% \fi
%\IEEEpeerreviewmaketitle

\section{Introduction}

    Online kernel selection (OKS) aims to dynamically select a kernel function
    or a reproducing kernel Hilbert space (RKHS)
    for online kernel learning algorithms.
    Compared to offline kernel selection,
    OKS poses more computational challenges for kernel selection algorithms,
    as both the kernel selection procedure and prediction need to occur in real-time.
    We can formulate it as a sequential decision problem.
    Let $\mathcal{K}=\{\kappa_i\}^K_{i=1}$ contain $K$ base kernel functions
    and $\mathcal{H}_i$ be a dense subset of the RKHS induced by $\kappa_i$.
    At round $t$,
    an adversary selects an instances ${\bm x}_t\in\mathbb{R}^d$.
    Then a learner chooses a hypothesis $f_t\in\cup^K_{i=1}\mathcal{H}_i$ and makes a prediction.
    The adversary gives the true output $y_t$,
    and the learner suffers a loss $\ell(f_t({\bm x}_t),y_t)$.
    We use the regret to measure the performance of the learner
    which is defined as
    $$
        \forall \kappa_i\in\mathcal{K},\quad\forall f\in \mathcal{H}_i,\quad\mathrm{Reg}(f)=
        \sum^T_{t=1}\ell(f_{t}({\bm x}_t),y_t) -\sum^T_{t=1}\ell(f({\bm x}_t),y_t).
    $$
    An efficient algorithm should guarantee $\mathrm{Reg}(f)=o(T)$.
    For convex loss functions,
    many algorithms \cite{Sahoo2014Online,Foster2017Parameter,Liao2021High}
     achieve (or imply)
    $$
        \mathrm{Reg}(f)=\tilde{O}\left(\left(\Vert f\Vert^\alpha_{\mathcal{H}_i}+1\right)\sqrt{T\ln{K}}\right),\alpha\in\{1,2\}.
    $$
    The algorithms adapt to the optimal,
    yet unknown hypothesis space within a small information-theoretical cost.

    A major challenge in OKS is the so called curse of dimensionality,
    that is, algorithms must store the previous $t-1$ instances at the $t$-th round.
    The $O(t)$ memory cost is prohibitive for large-scale online learning problems.
    To address this issue,
    we limit the memory of algorithms to a budget of $\mathcal{R}$ quanta.
    For convex loss functions,
    the worst-case regret is
    $\Theta\left(\max\left\{\sqrt{T},\frac{T}{\sqrt{\mathcal{R}}}\right\}\right)$
    \cite{Li2022Worst}
    which proves a trade-off between memory constraint and regret.
    To be specific, achieving a $O(T^{\alpha})$ regret
    requires $\mathcal{R}=\Omega\left(T^{2(1-\alpha)}\right)$, $\alpha\in\left[\frac{1}{2},1\right)$.
    Neither a constant nor a $\Theta(\ln{T})$ memory cost can guarantee sub-linear regret,
    thereby learning is impossible in any $\mathcal{H}_i$.
    We use the regret to define the learnability of a hypothesis space
    (see Definition \ref{JCST:def:learnability}).
    However,
    empirical results have shown that a small memory cost is enough to achieve good prediction performance
    \cite{Dekel2008The,Zhao2012Fast,Zhang2018Online}.
    It seems that learning is possible in the optimal hypothesis space.
    To fill the gap between learnability and regret,
    it is necessary to establish data-dependent regret bounds.
    To this end,
    we focus on the following two questions:
    \begin{enumerate}[\textbf{Q}1]
      \item What are the new trade-offs between memory constraints and regret,
            that is,
            how do certain data-dependent regret bounds depend on $\mathcal{R}$, $T$, and $K$?
      \item Sub-linear regret implies that some hypothesis spaces are learnable.
            Is it possible to achieve sub-linear regret
            within a $O(\ln{T})$ memory cost?
    \end{enumerate}

    In this paper,
    we will answer the two questions affirmatively.
    We first propose an algorithmic framework.
    Then we apply the algorithmic framework to two types of loss functions
    and achieve two kinds of data-dependent regret bounds.
    To maintain the memory constraint,
    we reduce it to the budget that limits the size of the buffer.
    Our algorithmic framework will use the buffer to store a subset of the observed examples.
    The main results are summarized as follows.
    \begin{enumerate}[1)]
      \item For the hinge loss function,
            our algorithm enjoys an expected kernel alignment regret bound as follows
            (see Theorem \ref{thm:JCST2022:regret-bound:M-OMD-H} for detailed result),
            \begin{equation}
            \label{eq:JCST2022:simplified_kernel_alignment_bound}
                \forall\kappa_i\in\mathcal{K},\forall f\in\mathbb{H}_i,\quad\mathbb{E}[\mathrm{Reg}(f)]=
                \tilde{O}\left(\sqrt{L_{T}(f)\ln{K}}+
                \frac{\sqrt{\mathcal{R}}}{\sqrt{K}}+
                \frac{\sqrt{K}}{\sqrt{\mathcal{R}}}\mathcal{A}_{T,\kappa_i}\right),
            \end{equation}
            where $\mathcal{A}_{T,\kappa_i}$ is called kernel alignment,
            $\mathbb{H}_i\subseteq \mathcal{H}_i$ and
            \begin{equation}
            \label{eq:JCST2023:cumulative_loss}
                L_{T}(f)=\sum^T_{t=1}\ell(f({\bm x}_t),y_t).
            \end{equation}
      \item For smooth loss functions
            (see Assumption \ref{ass:JCST2022:property_smooth_loss}),
            our algorithm achieves a high-probability small-loss bound
            (see Theorem \ref{thm:JCST2022:small_loss_bound:M-OMD-S} for detailed result).
            $\forall \kappa_i\in\mathcal{K},\forall f\in \mathbb{H}_i$,
            with probability at least $1-\delta$,
            \begin{equation}
            \label{eq:JCST2022:simplified_small-loss_bound}
            \begin{split}
                \mathrm{Reg}(f)=
                \tilde{O}\left(\frac{L_{T}(f)}{\sqrt{\mathcal{R}}}+
                \sqrt{L_{T}(f)\ln{K}}+\sqrt{\mathcal{R}}\right).
            \end{split}
            \end{equation}
      \item For smooth loss functions,
            we prove a lower bound on the regret
            (see Theorem \ref{thm:JCST2022:lower_bound_small_loss} for detailed result).
            Let $K=1$.
            The regret of any algorithm that stores $B$ examples
            satisfies
            $$
                \exists f\in\mathbb{H},\quad
                \mathbb{E}\left[\mathrm{Reg}(f)\right]=\Omega\left(U\cdot\frac{L_T(f)}{\sqrt{B}}\right),
            $$
            in which $U>0$ is a constant and bounds the norm of all hypotheses in $\mathbb{H}$.
            The upper bound in \eqref{eq:JCST2022:simplified_small-loss_bound} is optimal in terms of the dependence on $L_{T}(f)$.
    \end{enumerate}

    The data-dependent bounds in
    \eqref{eq:JCST2022:simplified_kernel_alignment_bound} and
    \eqref{eq:JCST2022:simplified_small-loss_bound}
    improve the bound
    $O\left(\sqrt{T\ln{K}}+\Vert f\Vert^2_{\mathcal{H}_i}\max\{\sqrt{T},\frac{T}{\sqrt{\mathcal{R}}}\}\right)$
    \cite{Li2022Worst},
    given that $\mathcal{A}_{T,\kappa_i}=O(T)$ and $L_T(f)=O(T)$.
    If $\kappa_i$ matches well with the data,
    then we expect $\mathcal{A}_{T,\kappa_i}\ll T$ or $L_T(f)\ll T$.
    In the worst case, i.e.,
    $\mathcal{A}_{T,\kappa_i}=\Theta(T)$ or $L_T(f)=\Theta(T)$,
    our bounds are same with previous result.
    We give new trade-offs between memory constraint and regret
    which answers \textbf{Q}1.
    For online kernel selection,
    we just aim to adapt to the optimal kernel $\kappa^\ast\in\mathcal{K}$.
    If $\kappa^\ast$ matches well with the data,
    then we expect $\mathcal{A}_{T,\kappa^\ast}=o(T)$ and $\min_{f\in\mathcal{H}_{\kappa^\ast}}L_T(f)=o(T)$.
    In this case,
    a $\Theta(\ln{T})$ memory cost is enough to achieve sub-linear regret
    which answers \textbf{Q}2.
    Thus learning is possible in $\mathbb{H}_{\kappa^\ast}$ within a small memory constraint.

    Our algorithmic framework reduces OKS to prediction with expert advice
    and uses the well-known optimistic mirror descent framework
    \cite{Chiang2012Online,Rakhlin2013Online}.
    We also propose a new buffer maintaining framework.

\section{Related Work}

    Previous work has adopted buffer maintaining technique and random features
    to develop algorithms for online kernel selection with memory constraint
    \cite{Sahoo2014Online,Liao2021High,Zhang2018Online,Shen2019Random}.
    However, all of existing results can not answer \textbf{Q}1 and \textbf{Q}2 well.
    The sketch-based online kernel selection algorithm \cite{Zhang2018Online}
    enjoys a $O(B\ln{T})$ regularized regret
    (the loss function incorporates a regularizer) within a buffer of size $B$.
    The regret bound becomes $O(T\ln{T})$ in the case of $B=\Theta(T)$.
    The Raker algorithm \cite{Shen2019Random} achieves a $O\left(\sqrt{T\ln{K}}+\frac{T}{\sqrt{D}}\right)$
    \footnote{The original regret bound is $O(U\sqrt{T\ln{K}}+\epsilon TU)$,
    and holds with probability $1-\Theta\left(\epsilon^{-2}\exp\left(-\frac{D}{4d+8}\epsilon^2\right)\right)$.}
    regret and suffers a space complexity in $O((d+K)D)$,
    where $D$ is the number of random features.
    The upper bound also shows a trade-off in the worst case.
    If $D$ is a constant or $D=\Theta(\ln{T})$,
    then Raker can not achieve a sub-linear regret bound,
    providing a negative answer to \textbf{Q}2.

    Our work is also related to achieving data-dependent regret bounds
    for online learning and especially online kernel learning.
    For online learning,
    various data-dependent regret bounds including
    small-loss bound \cite{Cesa-Bianchi2006Prediction,Lykouris2018Small,Lee2020Bias},
    variance bound \cite{Hazan2009Better,Hazan2010Extracting}
    and path-length bound
    \cite{Chiang2012Online,Steinhardt2014Adaptivity,Wei2018More},
    have been established.
    The adaptive online learning framework \cite{Foster2015Adaptive}
    can achieve data-dependent and model-dependent regret bounds,
    but can not induce a computationally efficient algorithm.
    For online kernel learning,
    it is harder to achieve data-dependent regret bounds,
    as we must balance the computational cost.
    For loss functions satisfying specific smoothness
    (see Assumption \ref{ass:JCST2022:property_smooth_loss}),
    the OSKL algorithm \cite{Zhang2013Online} achieves a small-loss bound.
    For loss functions with a curvature property,
    the PROS-N-KONS algorithm \cite{Calandriello2017Efficient}
    and the PKAWV algorithm \cite{Jezequel2019Efficient}
    achieve regret bounds depending on the effective dimension of the kernel matrix.
    None of these algorithms took memory constraints into account,
    failing to provide a trade-off between memory constraint and regret.

\section{Problem Setup}

    Let $\mathcal{I}_T:=\{({\bm x}_t,y_t)\}_{t\in[T]}$ be a sequence of examples,
    where ${\bm x}_t\in\mathcal{X}\subseteq\mathbb{R}^d, y_t\in [-1,1]$
    and $[T] := \{1,\ldots,T\}$.
    Let $\kappa(\cdot,\cdot):\mathbb{R}^d \times \mathbb{R}^d \rightarrow \mathbb{R}$
    be a positive semidefinite kernel function
    and $\mathcal{H}_{\kappa}$ be a dense subset of the associated RKHS,
    such that, for any $f\in\mathcal{H}_{\kappa}$,
    \begin{enumerate}[(i)]
      \item $\langle f,\kappa({\bm x},\cdot)\rangle_{\mathcal{H}_{\kappa}}=f({\bm x}),
    \forall {\bm x}\in\mathcal{X}$,
      \item $\mathcal{H}_{\kappa}=\overline{\mathrm{span}(\kappa({\bm x}_t,\cdot)\vert t\in[T])}$.
    \end{enumerate}
    We define $\langle\cdot,\cdot\rangle_{\mathcal{H}_{\kappa}}$ as the inner product in $\mathcal{H}_{\kappa}$,
    which induces the norm $\Vert f\Vert_{\mathcal{H}_{\kappa}}=\sqrt{\langle f,f\rangle_{\mathcal{H}_{\kappa}}}$.
    For simplicity,
    we will omit the subscript $\mathcal{H}_{\kappa}$ in the inner product.
    Let
    $\ell(\cdot,\cdot):\mathbb{R}\times[-1,1]\rightarrow \mathbb{R}$
    be the loss function.
    Denote by
    $$
        \mathcal{B}_{\psi}(f,g)
        =\psi(f)-\psi(g)-\langle\nabla\psi(g),f-g\rangle,~\forall f,g\in \mathcal{H}_{\kappa},
    $$
    the Bregman divergence associated with a strongly convex regularizer
    $\psi(\cdot):\mathcal{H}_{\kappa}\rightarrow \mathbb{R}$.

    Let $\mathcal{K}=\{\kappa_i\}^K_{i=1}$ contain $K$ base kernels.
    If an oracle gives the optimal kernel $\kappa^\ast\in\mathcal{K}$ for $\mathcal{I}_T$,
    then we just run an online kernel learning algorithm in $\mathcal{H}_{\kappa^{\ast}}$.
    Since $\kappa^\ast$ is unknown,
    the learner aims to develop a kernel selection algorithm
    and generate a sequence of hypotheses $\{f_t\}^T_{t=1}$
    that can compete with any $f\in\mathcal{H}_{\kappa^{\ast}}$,
    that is, the learner hopes to guarantee $\mathrm{Reg}(f)=o(T)$.
    For any $f\in\mathcal{H}_i$,
    we have $f=\sum^T_{t=1}a_t\kappa_i({\bm x}_t,\cdot)$.
    Thus storing $f_t$ needs to store $\{({\bm x}_\tau,y_\tau)\}^{t-1}_{\tau=1}$,
    incurring a memory cost in $O(t)$.
    In this paper,
    we consider online kernel selection with memory constraint.
    Next we define the memory constraint
    and reduce it to the size of a buffer.

    \begin{definition}[Memory Budget \cite{Li2022Worst}]
    \label{def:JCST2022:memory_budget}
        A memory budget of $\mathcal{R}$ quanta is the maximal memory
        that any online kernel selection algorithm can use.
    \end{definition}

    \begin{assumption}[\cite{Li2022Worst}]
    \label{ass:JCST2022:reduction}
        $\forall \kappa_i\in\mathcal{K}$,
        there is a constant $\alpha>0$,
        such that any budgeted online kernel leaning algorithm running in $\mathcal{H}_i$
        can maintain a buffer of size $B\leq\alpha\mathcal{R}$
        within $\mathcal{R}$ quanta memory constraint.
        If the space complexity of algorithm is linear with $B$,
        then ``$=$'' holds.
    \end{assumption}

    Assumption \ref{ass:JCST2022:reduction}
    implies that there is no need to assume $\mathcal{R}=\infty$ unless $T=\infty$,
    as the maximal value of $\mathcal{R}$ satisfies $B=T$.
    The budgeted online kernel learning algorithms are such algorithms
    that operate on a subset of the observed examples,
    such as Forgetron \cite{Dekel2008The}, BSGD \cite{Wang2012Breakingecond}
    to name but a few.
    In Assumption \ref{ass:JCST2022:reduction},
    $\alpha$ is independent of kernel function,
    since the memory cost is used to store the examples and the coefficients.

    \begin{assumption}
    \label{ass:JCST2022:bounded_kernel}
        There is a constant $k_1>0$, such that
        $\forall \kappa_i\in\mathcal{K}$ and $\forall {\bm u}\in\mathcal{X}$,
        $\kappa_i({\bm u},{\bm u}) \in[k_1,1]$.
    \end{assumption}

    We will define two types of data complexities and use them to bound the regret.
    The first one is called kernel alignment defined as follows
    $$
        \mathcal{A}_{T,\kappa_i}=\sum^T_{t=1}\kappa_i({\bm x}_t,{\bm x}_t)
    -\frac{1}{T}{\bm Y}^{\mathrm{T}}{\bm K}_i{\bm Y},
    $$
    where ${\bm Y}=(y_1,\ldots,y_T)^{\mathrm{T}}$
    and ${\bm K}_i$ is the kernel matrix induced by $\kappa_i$.
    $\mathcal{A}_{T,\kappa_i}$ quantifies how well ${\bm K}_i$ matches the examples.
    If ${\bm K}_i={\bm Y}{\bm Y}^{\mathrm{T}}$ is the ideal kernel matrix,
    then $\mathcal{A}_{T,\kappa_i}=\Theta(1)$.
    The second one is called small-loss defined by $\min_{f\in\mathbb{H}_i}L_T(f)$
    where $L_T(f)$ follows \eqref{eq:JCST2023:cumulative_loss} and
    $$
        \mathbb{H}_i=\left\{f\in\mathcal{H}_i:\Vert f\Vert_{\mathcal{H}_i}\leq U\right\}.
    $$
    Both $\mathcal{A}_{T,\kappa_i}$ and $L_T(f)$ are independent of algorithms.

    Finally,
    we define the online learnability which is a variant of the learnability defined in
    \cite{Rakhlin2010Online}.
    \begin{definition}[Online Learnability]
    \label{JCST:def:learnability}
        Given $\mathcal{I}_T$, a hypothesis space $\mathcal{H}$ is said to be online learnable
        if $\lim_{T\rightarrow\infty}\frac{\mathrm{Reg}(f)}{T}=0$, $\forall f\in \mathcal{H}$.
    \end{definition}
    The definition is equivalent to $\mathrm{Reg}(f)=o(T)$.
    We can also replace $\mathrm{Reg}(f)$ with $\mathbb{E}[\mathrm{Reg}(f)]$.
    The learnability defined in \cite{Rakhlin2010Online} holds for the entire sample space,
    i.e., the worst-case examples,
    while our definition only considers a fixed $\mathcal{I}_T$.
    In the worst-case,
    the two kinds of learnability are equivalent.
    Given that $\mathcal{I}_T$ may not always be generated in the worst-case,
    our learnability can adapt to the hardness of $\mathcal{I}_T$.

\section{Algorithmic Framework}

    We reduce OKS to prediction with expert advice (PEA)
    where $\kappa_i$ corresponds to the $i$-th expert.
    The main idea of our algorithmic framework is summarized as follows:
    (i) generating $\{f_{t,i}\in\mathcal{H}_i\}^T_{t=1}$ for all $i\in[K]$;
    (ii) aggregating the predictions $\{f_{t,i}({\bm x}_t)\}^K_{i=1}$.
    The challenge is how to control the size of $\{f_{t,i}\}^K_{i=1}$.
    To this end,
    we propose a new buffer maintaining framework
    containing an adaptive sampling and a removing process.

\subsection{Adaptive Sampling}

    For all $i\in[K]$,
    we use the optimistic mirror descent framework (OMD) \cite{Chiang2012Online,Rakhlin2013Online}
    to generate $\{f_{t,i}\}^T_{t=1}$.
    OMD maintains two sequences of hypothesis, i.e.,
    $\{f_{t,i}\}^T_{t=1}$ and $\{f'_{t,i}\}^T_{t=1}$ which are defined as follows,
    \begin{align}
        f_{t,i}=&\mathop{\arg\min}_{f\in\mathcal{H}_i}
        \left\{\left\langle f,\hat{\nabla}_{t,i}\right\rangle+\mathcal{B}_{\psi_i}\left(f,f'_{t-1,i}\right)\right\},
        \label{eq:JCST2022:OMD_first_updating}\\
        f'_{t,i}=&\mathop{\arg\min}_{f\in\mathbb{H}_i}
        \left\{\left\langle f,\nabla_{t,i}\right\rangle+\mathcal{B}_{\psi_i}\left(f,f'_{t-1,i}\right)\right\},
        \label{eq:JCST2022:OMD_second_updating}
    \end{align}
    where $\nabla_{t,i}$ is the (sub)-gradient of $\ell(f_{t,i}({\bm x}_t),y_t)$,
    $\hat{\nabla}_{t,i}$ is an optimistic estimator of $\nabla_{t,i}$ and
    \begin{align*}
        \psi_i(f)=&\frac{1}{2\lambda_i}\Vert f\Vert^2_{\mathcal{H}_i},\\
     \nabla_{t,i}=&\ell'(f_{t,i}({\bm x}_t),y_t)\cdot\kappa_i({\bm x}_t,\cdot).
    \end{align*}
    Note that the hypothesis space in \eqref{eq:JCST2022:OMD_first_updating}
    and \eqref{eq:JCST2022:OMD_second_updating}
    is $\mathcal{H}_i$ and $\mathbb{H}_i$, respectively.
    This trick \cite{Li2023Improved}
    is critical to obtaining regret bounds depending on kernel alignment.
    At the beginning of round $t$,
    we execute \eqref{eq:JCST2022:OMD_first_updating} and compute $f_{t,i}({\bm x}_t)$.
    The key is to control the size of $f'_{t,i}$.
    By Assumption \ref{ass:JCST2022:reduction},
    we construct a buffer denoted by $S_i$ that stores the examples constructing $f'_{t,i}$.
    Thus we just need to limit the size of $S_i$.
    To this end,
    we propose an adaptive sampling scheme.
    Let $b_{t,i}$ be a Bernoulli random variable satisfying
    \begin{equation}
    \label{eq:JCST2022:sampling_probability}
        \mathbb{P}[b_{t,i}=1]
        =\frac{\Vert \nabla_{t,i}-\hat{\nabla}_{t,i}\Vert^{\nu}_{\mathcal{H}_i}}{Z_t},\quad \nu\in\{1,2\},
    \end{equation}
    where $Z_t$ is a normalizing constant.
    We further define
    \begin{align*}
        \mathrm{con} (a(i)):=&
        \left\Vert\kappa_{i}({\bm x}_{i(s_t)},\cdot)-\kappa_i({\bm x}_t,\cdot)\right\Vert_{\mathcal{H}_i}
        \leq\gamma_{t,i},\\
        {\bm x}_{i(s_t)}=&\mathop{\arg\min}_{{\bm x}_\tau\in S_i}
        \left\Vert\kappa_{i}({\bm x}_{\tau},\cdot)-\kappa_i({\bm x}_t,\cdot)\right\Vert_{\mathcal{H}_i}.
    \end{align*}
    $\mathrm{con} (a(i))$ means that if there is an instance in $S_i$
    that is similar with ${\bm x}_t$,
    then we can use it as a proxy of ${\bm x}_t$.
    In this way, $({\bm x}_t,y_t)$ will not be added into $S_i$.

    If $\nabla_{t,i}= 0$, then $S_i$ keeps unchanged
    and we execute \eqref{eq:JCST2022:OMD_second_updating}.
    If $\nabla_{t,i}\neq 0$ and $\mathrm{con} (a(i))$,
    then $S_i$ keeps unchanged and
    we replace \eqref{eq:JCST2022:OMD_second_updating}
    with \eqref{eq:JCST2022:OMD_second_updating_2}.
    \begin{equation}
    \label{eq:JCST2022:OMD_second_updating_2}
        f'_{t,i}
        =\mathop{\arg\min}_{f\in\mathbb{H}_i}\left\{\left\langle f,\nabla_{i(s_t),i}\right\rangle+\mathcal{B}_{\psi_{i}}\left(f,f'_{t-1,i}\right)\right\}.
    \end{equation}
    If $\nabla_{t,i}\neq 0,\neg \mathrm{con} (a(i))$,
    then we replace \eqref{eq:JCST2022:OMD_second_updating}
    with \eqref{eq:JCST2022:OMD_second_updating_3}.
    \begin{equation}
    \label{eq:JCST2022:OMD_second_updating_3}
        f'_{t,i}=\mathop{\arg\min}_{f\in\mathbb{H}_i}
        \left\{\left\langle f,\tilde{\nabla}_{t,i}\right\rangle+\mathcal{B}_{\psi_{i}}\left(f,f'_{t-1,i}\right)\right\}.
    \end{equation}
    In this case,
    if $b_{t,i}=1$,
    then we add $({\bm x}_t,y_t)$ into $S_i$.
    In the above two equations, we define
    \begin{align*}
         \nabla_{i(s_t),i}=&\ell'(f_{t,i}({\bm x}_t),y_t)\cdot\kappa_i({\bm x}_{i(s_t)},\cdot),\\
         \tilde{\nabla}_{t,i}=&
         \frac{\nabla_{t,i}-\hat{\nabla}_{t,i}}{\mathbb{P}[b_{t,i}=1]}\cdot\mathbb{I}_{b_{t,i}=1}+\hat{\nabla}_{t,i}.
    \end{align*}

\subsection{Removing a Half of Examples}

    We allocate a memory budget $\mathcal{R}_i$ for $f'_{t,i}$
    and require $\sum^K_{i=1}\mathcal{R}_i\leq \mathcal{R}$.
    By Assumption \ref{ass:JCST2022:reduction},
    we just need to limit $\vert S_i\vert\leq \alpha\mathcal{R}_i$.
    At any round $t$,
    if $\vert S_i\vert=\alpha\mathcal{R}_i$ and $b_{t,i}=1$,
    then we remove a half of examples from $S_i$ \cite{Li2023Improved}.
    Let $S_i=\{({\bm x}_{r_j},y_{r_j})\}^{\alpha\mathcal{R}_i}_{j=1}$.
    We rewrite $f'_{t-1,i}$ as follows,
    \begin{align*}
           f'_{t-1,i}=&f'_{t-1,i}(1)+f'_{t-1,i}(2),\\
        f'_{t-1,i}(1)=&\sum^{\alpha\mathcal{R}_i/2}_{j=1}\beta_{r_j}\kappa_i({\bm x}_{r_j},\cdot),\\
        f'_{t-1,i}(2)
        =&\sum^{\alpha\mathcal{R}_i}_{j=\alpha\mathcal{R}_i/2+1}\beta_{r_j}\kappa_i({\bm x}_{r_j},\cdot).
    \end{align*}
    Let $o_t\in\{1,2\}$ and $q_t\in\{1,2\}\setminus \{o_t\}$.
    We will remove $f'_{t-1,i}(o_t)$ from $f'_{t-1,i}$, that is,
    the examples constructing $f'_{t-1,i}(o_t)$ are removed from $S_i$.
    We further project $f'_{t-1,i}(q_t)$ onto $\mathbb{H}_i$ by
    $$
        \bar{f}'_{t-1,i}(q_t)=\mathop{\arg\min}_{f\in\mathbb{H}_i}
        \left\Vert f- f'_{t-1,i}(q_t)\right\Vert^2_{\mathcal{H}}.
    $$
    Then we execute the second mirror descent and
    redefine \eqref{eq:JCST2022:OMD_second_updating_3} as follows
    \begin{equation}
    \label{eq:JCST2022:OMD_second_updating_4}
        f'_{t,i}=\mathop{\arg\min}_{f\in\mathbb{H}_i}
        \left\{\left\langle f,\tilde{\nabla}_{t,i}\right\rangle+\mathcal{B}_{\psi_{i}}\left(f,\bar{f}'_{t-1,i}(q_t)\right)\right\}.
    \end{equation}

    Our buffer maintaining framework always ensures $\vert S_i\vert\leq \alpha\mathcal{R}_i$.
    A natural approach is the restart technique
    which removes all of the examples in $S_i$, i.e., resetting $S_i=\emptyset$.
    We will prove that removing a half of examples can achieve better regret bounds.

\subsection{Reduction to PEA}

    Similar to previous online multi-kernel learning algorithms \cite{Sahoo2014Online,Shen2019Random,Jin2010Online},
    we use PEA framework  \cite{Cesa-Bianchi2006Prediction} to
    aggregate $K$ predictions.
    Let $\Delta_{K-1}$ be the $K-1$ dimensional simplex.
    At each round $t$,
    we maintain a probability distribution ${\bm p}_t\in\Delta_{K-1}$ over $\{f_{t,i}\}^K_{i=1}$.
    and output the prediction $f_{t}({\bm x}_t)=\sum^K_{i=1}p_{t,i}f_{t,i}({\bm x}_t)$
    or $\hat{y}_t=\mathrm{sign}(f_{t}({\bm x}_t)$.
    By the multiplicative-weight update in Subsection 2.1 \cite{Cesa-Bianchi2006Prediction},
    ${\bm p}_{t+1}$ is defined as follows
    \begin{equation}
    \label{eq:JCST2022:updating_sampling_probability}
    \begin{split}
        p_{t+1,i}&= \frac{w_{t+1,i}}{\sum^K_{j=1}w_{t+1,j}},\\
        w_{t+1,i}&=\exp\left(-\eta_{t+1}\sum^t_{\tau=1}c_{\tau,i}\right),\\
        \eta_{t+1}=&\frac{\sqrt{2\ln{K}}}{\sqrt{1+\sum^{t}_{\tau=1}\sum^K_{i=1}p_{\tau,i}c^2_{\tau,i}}},
        \eta_1=\sqrt{2\ln{K}},
    \end{split}
    \end{equation}
    where $c_{t,i}=g(f_{t,i}({\bm x}_t),y_t)$.
    $g(\cdot,\cdot):\mathbb{R}^2\rightarrow \mathbb{R}^{+}\cup\{0\}$ will be defined later.

\section{Applications}

    In this section,
    we apply the proposed algorithmic framework to two types of loss functions
    and derive two kinds of data-dependent regret bounds.

\subsection{the Hinge Loss Function}

    We consider online binary classification tasks.
    Let the loss function be the Hinge loss.
    Thus $\nabla_{t,i}=-y_t\kappa_i({\bm x}_t,\cdot)\cdot\mathbb{I}_{y_tf_{t,i}({\bm x}_t)<1}$.
    We use the ``Reservoir Sampling (RS)'' technique \cite{Hazan2009Better,Vitter1985Random}
    to create $\hat{\nabla}_{t,i}$.
    To this end,
    we create a new buffer $V_t$ with size $M\geq 1$.
    Initializing $\hat{\nabla}_{1,i}=0$, and for $t\geq 2$,
    we define
    $$
        \hat{\nabla}_{t,i}=-\frac{1}{\vert V_t\vert}\sum_{({\bm x},y)\in V_t}y\cdot \kappa_{i}({\bm x},\cdot).
    $$
    RS is defined follows.
    At the end of round $t\geq 1$,
    we add $({\bm x}_t,y_t)$ into $V_t$ with probability $\min\left\{1,\frac{M}{t}\right\}$.
    If $\vert V_t\vert=M$
    and we will add the current example,
    then an old example must be removed from $V_t$ uniformly.
    We create another buffer $S_0$
    that stores all of the examples added into $V_t$.
    It is easy to prove that $\mathbb{E}[\vert S_0\vert] \leq M(1+\ln{T})$.
    Let $\mathcal{R}_0=\frac{M(1+\ln{T})}{\alpha}$
    and $\mathcal{R}_{i}= \frac{\mathcal{R}-\mathcal{R}_0}{K}$
    where we require $\mathcal{R}\geq 2\mathcal{R}_0$.
    It is easy to verify that $\sum^K_{i=0}\mathcal{R}_i=\mathcal{R}$.

    We instantiate the sampling scheme \eqref{eq:JCST2022:sampling_probability} as follows,
    $$
        \nu =2,\quad Z_t=\left\Vert \nabla_{t,i}-\hat{\nabla}_{t,i}\right\Vert^2_{\mathcal{H}_i}+
        \left\Vert\hat{\nabla}_{t,i}\right\Vert^2_{\mathcal{H}_i}.
    $$
    The additional term $\left\Vert\hat{\nabla}_{t,i}\right\Vert^2_{\mathcal{H}_i}$ in $Z_t$
    is important for controlling the times of removing operation,
    while it also increases the regret.
    The trick that uses different hypothesis spaces in \eqref{eq:JCST2022:OMD_first_updating}
    and \eqref{eq:JCST2022:OMD_second_updating},
    gives a negative term in the regret bound.
    The negative term can cancel out the increase on the regret.
    Let $o_t=2$.
    We remove $f'_{t-1,i}(2)$ from $f_{t,i}$.
    Then \eqref{eq:JCST2022:OMD_second_updating_4} becomes
    \begin{equation}
    \label{eq:JCST2022:OMD_second_updating_5}
        f'_{t,i}=\mathop{\arg\min}_{f\in\mathbb{H}_i}
        \left\{\left\langle f,\tilde{\nabla}_{t,i}\right\rangle
        +\mathcal{B}_{\psi_{i}}\left(f,\bar{f}'_{t-1,i}(1)\right)\right\}.
    \end{equation}

    At the $t$-th round,
    our algorithm outputs the prediction $\hat{y}_t=\mathrm{sign}(f_{t}({\bm x}_t))$.
    Let ${\bm p}_1$ be the uniform distribution
    and ${\bm p}_{t+1}$ follow \eqref{eq:JCST2022:updating_sampling_probability}
    where $c_{t,i}=\ell(f_{t,i}({\bm x}_t),y_t)$.
    We name this algorithm M-OMD-H (Memory bounded OMD for the Hinge loss)
    and present the pseudo-code
    in Algorithm \ref{alg:JCST2022:M-OMD-H}.

    \begin{algorithm}[t]
        \caption{\small{M-OMD-H}}
        \footnotesize
        \label{alg:JCST2022:M-OMD-H}
        \begin{algorithmic}[1]
        \REQUIRE{$\{\lambda_i\}^K_{i=1}$, $\{\eta_t\}^T_{t=1}$, $\mathcal{R}$, $U$.}
        \ENSURE{$f'_{0,i}=\hat{\nabla}_{1,i}=0$, $S_i=V=\emptyset$.}
        \FOR{$t=1,2,\ldots,T$}
            \STATE Receive ${\bm x}_t$\;
            \FOR{$i=1,\ldots,K$}
                \STATE Find ${\bm x}_{i(s_t)}$
                \STATE Compute $\hat{\nabla}_{t,i}=\frac{-1}{\vert V\vert}\sum_{({\bm x},y)\in V}
                y\kappa_i({\bm x},\cdot)$
                \STATE Compute $f_{t,i}({\bm x}_t)$ according to \eqref{eq:JCST2022:OMD_first_updating}
            \ENDFOR
            \STATE Output $\hat{y}_t=\mathrm{sign}(f_{t}({\bm x}_t))$ and receive $y_t$\;
            \FOR{$i = 1,\ldots,K$}
                \IF{$y_tf_{t,i}({\bm x}_t)<1$}
                    \IF{$\mathrm{con} (a(i))$}
                        \STATE Update $f'_{t,i}$ following \eqref{eq:JCST2022:OMD_second_updating_2}
                    \ELSE
                        \STATE Sample $b_{t,i}$
                        \IF{$b_{t,i}=1$ and $\vert S_i\vert\leq \alpha\mathcal{R}_i-1$}
                            \STATE Update $f'_{t,i}$ following \eqref{eq:JCST2022:OMD_second_updating_3}
                        \ENDIF
                        \IF{$b_{t,i}=1$~$\mathrm{and}$~$\vert S_i\vert= \alpha\mathcal{R}_i$}
                            \STATE Compute $\bar{f}'_{t-1,i}(1)$
                            \STATE Update $f'_{t,i}$ following \eqref{eq:JCST2022:OMD_second_updating_5}
                            \STATE Remove the latest $\frac{\alpha\mathcal{R}_i}{2}$ examples from $S_i$
                        \ENDIF
                        \STATE Update $S_i=S_i\cup \{({\bm x}_t,y_t):b_{t,i}=1\}$
                    \ENDIF
                \ENDIF
            \STATE Compute ${\bm p}_{t+1}$ by \eqref{eq:JCST2022:updating_sampling_probability}
            \STATE Update reservoir $V$ and $S_0$
            \ENDFOR
        \ENDFOR
        \end{algorithmic}
    \end{algorithm}

    \begin{lemma}
    \label{lem:JCST2022:number_of_epoch:hinge_loss}
        Let $M> 1$ and $\alpha\mathcal{R}:=B\geq 2M(1+\ln{T})$.
        For any $\mathcal{I}_T$,
        the expected times that M-OMD-H executes removing operation on $S_i$ are
        $\left\lceil\frac{4K\tilde{\mathcal{A}}_{T,\kappa_i}}{Bk_1}\right\rceil$
        at most, in which
        $$
            \tilde{\mathcal{A}}_{T,\kappa_i}
            =1+\sum^{T}_{t=2}\left\Vert y_t\kappa_{i}({\bm x}_t,\cdot)
            -\frac{\sum_{({\bm x},y)\in V_t}y\kappa_{i}({\bm x},\cdot)}{\vert V_t\vert}\right\Vert^2_{\mathcal{H}_i}.
        $$
    \end{lemma}

    Note that that the times of removing operation depend on the kernel alignment,
    rather than $T$.

    \begin{theorem}
    \label{thm:JCST2022:regret-bound:M-OMD-H}
        Suppose $\mathcal{R}$ satisfies the condition in
        Lemma \ref{lem:JCST2022:number_of_epoch:hinge_loss}.
        Let $B>K$, $U>1$, $\lambda_i=U\frac{\sqrt{K}}{\sqrt{2B}}$,
        $\eta_{t}$ follow \eqref{eq:JCST2022:updating_sampling_probability}
        and
        \begin{equation}
        \label{eq:JCST2022:learning_rate}
            \gamma_{t,i}=\frac{\left\Vert \nabla_{t,i}-\hat{\nabla}_{t,i}\right\Vert^2_{\mathcal{H}_i}}
            {\sqrt{1+\sum_{\tau\leq t} \left\Vert \nabla_{\tau,i}-
            \hat{\nabla}_{\tau,i}\right\Vert^2_{\mathcal{H}_i}\cdot\mathbb{I}_{\nabla_{\tau,i}\neq 0}}}.
        \end{equation}
        The expected regret of M-OMD-H satisfies,
        $$c
            \forall\kappa_i\in\mathcal{K},f\in \mathbb{H}_i,\quad
            \mathbb{E}\left[\mathrm{Reg}(f)\right]
            = O\left(\sqrt{UL_T(f)\ln{K}}+\frac{U\sqrt{B}}{\sqrt{K}}
            +\frac{\sqrt{K}U\mathcal{A}_{T,\kappa_i}\ln{T}}{\sqrt{B}k_1}\right).
        $$
    \end{theorem}

    For any competitor $f$ satisfying $\Vert f\Vert_{\mathcal{H}_{i}}<1$,
    it must be $L_T(f)=\Theta(T)$,
    inducing a trivial upper bound.
    To address this issue,
    we must set $U>1$ and exclude such competitors.
    Theorem \ref{thm:JCST2022:regret-bound:M-OMD-H}
    reveals how the regret bound increases with $K$, $\mathcal{R}$, and $\mathcal{A}_{T,\kappa_i}$.
    The larger the memory budget is, the smaller the regret bound will be.
    Theorem \ref{thm:JCST2022:regret-bound:M-OMD-H} gives an answer to \textbf{Q}1.
    If the optimal kernel $\kappa^{\ast}\in\mathcal{K}$ matches well with the examples,
    i.e., $\mathcal{A}_{T,\kappa^{\ast}}=o(T)$,
    then it is possible to achieve a $o(T)$ regret bound
    in the case of $\mathcal{R}=\Theta(\ln{T})$.
    Such a result gives an answer to \textbf{Q}2.
    By Definition \ref{JCST:def:learnability},
    learning is possible in $\mathbb{H}_{\kappa^\ast}$.
    The information-theoretical cost that the optimal kernel is unknown
    is $O\left(\sqrt{UL_{T}(f)\ln{K}}\right)$,
    which is a lower order term.

    Our regret bound can also derive the state-of-the-art results.
    Let $B=\Theta(\mathcal{A}_{T,\kappa_i})$.
    Then we obtain a $O\left(U\sqrt{K\mathcal{A}_{T,\kappa_i}}\right)$ expected bound.
    Previous work proved a $O\left(U\sqrt{K}T^{\frac{1}{4}}\mathcal{A}^{\frac{1}{4}}_{T,\kappa_i}\right)$
    high-probability bound \cite{Liao2021High}.
    If $\mathcal{A}_{T,\kappa_i}=O(T)$,
    then we achieve an expected bound of $O\left(U\sqrt{K}\frac{T}{\sqrt{B}}\right)$,
    which matches the upper bound in \cite{Li2022Worst}.
    If $B=\Theta(T)$,
    then our regret bound matches the upper bounds in
    \cite{Sahoo2014Online,Foster2017Parameter,Shen2019Random}.

    Next, we present an algorithm-dependent bound
    demonstrating that removing half of the examples
    can be more effective than the restart technique.
    We introduction two new notations as follows
    \begin{align*}
        \mathcal{J}_i=&\{t\in [T]:\vert S_i\vert=\alpha\mathcal{R}, b_{t,i}=1\},\\
        \Lambda_{i}=&\sum_{t\in \mathcal{J}_i}
        \left[\left\Vert \bar{f}'_{t-1,i}(1)-f\right\Vert^2_{\mathcal{H}_i}
        -\left\Vert f'_{t-1,i}-f\right\Vert^2_{\mathcal{H}_i}\right].
    \end{align*}
    There must be a constant $\xi_i\in(0,4]$
    such that $\Lambda_i=\xi_i U^2\vert \mathcal{J}_i\vert$.
    Recalling that $\bar{f}'_{t-1,i}(1)$ is the initial hypothesis after a removing operation.
    It is natural that if $\left(\bar{f}'_{t-1,i}(1)-f\right)$ is close to $\left(f'_{t-1,i}-f\right)$,
    then $\xi_i\ll 4$.
    The restart technique sets $\bar{f}'_{t-1,i}(1)=0$
    implying $\Lambda_i\leq U^2\vert \mathcal{J}_i\vert$.
    In the worst case,
    our approach is slightly worse than the restart.
    If $\xi_i$ is sufficiently small,
    then our approach is much better than the restart.

    \begin{theorem}[Algorithm-dependent Bound]
    \label{thm:JCST2022:algorithm_dependent:M-OMD-H}
        Suppose the conditions in Theorem \ref{thm:JCST2022:regret-bound:M-OMD-H} are satisfied.
        For each $i\in[K]$,
        let $\lambda_i=\frac{\sqrt{2}U}{\sqrt{5\tilde{\mathcal{A}}_{T,\kappa_i}}}$.
        If $\xi_i\cdot\vert \mathcal{J}_i\vert\leq 1$ for all $i\in[K]$,
        then the expected regret of M-OMD-H satisfies,
        $$
            \forall\kappa_i\in\mathcal{K},f\in \mathbb{H}_i,\quad
            \mathbb{E}\left[\mathrm{Reg}(f)\right]=
            O\left(\sqrt{L_T(f)\ln{K}}+\frac{UK}{B}\mathcal{A}_{T,\kappa_i}\ln{T}+
            U\sqrt{\mathcal{A}_{T,\kappa_i}\ln{T}}\right).
        $$
    \end{theorem}

    The dominate term in Theorem \ref{thm:JCST2022:algorithm_dependent:M-OMD-H}
    is $\tilde{O}\left(\frac{UK\mathcal{A}_{T,\kappa_i}}{B}\right)$,
    while the dominate term in Theorem \ref{thm:JCST2022:regret-bound:M-OMD-H}
    is $\tilde{O}\left(\frac{UK\mathcal{A}_{T,\kappa_i}}{\sqrt{B}}\right)$.
    The restart technique only enjoys the regret bound given
    in Theorem \ref{thm:JCST2022:regret-bound:M-OMD-H}.
    Thus removing a half of examples can be better than the restart.

\subsection{Smooth Loss Functions}

    We first define the smooth loss functions.

    \begin{assumption}[\cite{Zhang2013Online}]
    \label{ass:JCST2022:property_smooth_loss}
        Let $G_1$ and $G_2$ be positive constants.
        For any $u$ and $y$,
        $\ell(u,y)$ satisfies
        (i) $\vert \ell'(u,y)\vert \leq G_1$,
        (ii) $\vert\ell'(u,y)\vert\leq G_2\ell(u,y)$,
        where $\ell'(u,y)=\frac{\mathrm{d}\,\ell(u,y)}{\mathrm{d}\,u}$.
    \end{assumption}
    The logistic loss function satisfies
    Assumption \ref{ass:JCST2022:property_smooth_loss} with $G_2=1$.
    For the hinge loss function,
    M-OMD-H uniformly allocates the memory budget and
    induces a $O\left(\sqrt{K}\right)$ factor in the regret bound.
    For smooth loss functions,
    we will propose a memory sharing scheme that can avoid the $O\left(\sqrt{K}\right)$ factor.

    Recalling that the memory is used to store the examples and coefficients.
    The main challenge is how to share the examples.
    To this end,
    we only keep a single buffer $S$.
    We first rewrite $f_t({\bm x}_t)$ as follows,
    $$
        f_t({\bm x}_t)
        =\sum^K_{i=1}p_{t,i}f_{t,i}({\bm x}_t)
        =\sum^K_{i=1}p_{t,i}
        \sum_{{\bm x}_{\tau}\in S_i}a_{\tau,i}\kappa_i({\bm x}_{\tau},{\bm x}_t).
    $$
    We just need to ensure $S_i=S$ for all $i\in[K]$.
    For each $f_{t,i}$,
    we define a surrogate gradient
    $\nabla_{t,i}=\ell'(f_{t}({\bm x}_t),y_t)\cdot\kappa_i({\bm x}_t,\cdot)$.
    Note that we use the first derivative $\ell'(f_{t}({\bm x}_t),y_t)$,
    rather than $\ell'(f_{t,i}({\bm x}_t),y_t)$.
    Let $\hat{\nabla}_{t,i}=0$ in \eqref{eq:JCST2022:OMD_first_updating}.
    Then we obtain a single update rule as follows
    $$
        f_{t+1,i}=\mathop{\arg\min}_{f\in\mathbb{H}_i}
        \left\{\langle f,\nabla_{t,i}\rangle+\mathcal{B}_{\psi_i}(f,f_{t,i})\right\}.
    $$
    To ensure $S_i=S$,
    we must change the sampling scheme \eqref{eq:JCST2022:sampling_probability}.
    Let $b_{t,i}=b_t$ and $\mathrm{con} (a(i))=\mathrm{con} (a)$ for all $i\in[K]$.
    We define $\nu=1$ and
    \begin{align*}
        Z_t=&\Vert \nabla_{t,i}\Vert_{\mathcal{H}_i}+G_1\sqrt{\kappa_i({\bm x}_t,{\bm x}_t)},\\
        \mathrm{con} (a):=&\max_{i\in[K]}
        \Vert\kappa_{i}({\bm x}_{i(s_t)},\cdot)-\kappa_i({\bm x}_t,\cdot)\Vert_{\mathcal{H}_i}
        \leq\gamma_t.
    \end{align*}
    If $\mathrm{con} (a)$, then
    we replace \eqref{eq:JCST2022:OMD_second_updating_2}
    with \eqref{eq:JCST2022:smooth_loss:surrogate_mirror_descent}.
    \begin{equation}
    \label{eq:JCST2022:smooth_loss:surrogate_mirror_descent}
    \begin{split}
        f_{t+1,i}
        =&\mathop{\arg\min}_{f\in\mathbb{H}_i}\left\{\left\langle f,\nabla_{i(s_t),i}\right\rangle
        +\mathcal{B}_{\psi_{i}}(f,f_{t,i})\right\},\\
        \nabla_{i(s_t),i}=&\ell'(f_t({\bm x}_t),y_t)\cdot\kappa_i\left({\bm x}_{i(s_t)},\cdot\right).
    \end{split}
    \end{equation}
    Otherwise,
    we replace \eqref{eq:JCST2022:OMD_second_updating_3}
    with \eqref{eq:JCST2022:smooth_loss:mirror_descent}.
    \begin{equation}
    \label{eq:JCST2022:smooth_loss:mirror_descent}
        f_{t+1,i}=\mathop{\arg\min}_{f\in\mathbb{H}_i}
        \left\{\left\langle f,\tilde{\nabla}_{t,i}\right\rangle
        +\mathcal{B}_{\psi_{i}}(f,f_{t,i})\right\}.
    \end{equation}
    If $\vert S\vert=\alpha\mathcal{R}$ and $b_t=1$,
    then let $o_t=1$.
    We remove $f_{t,i}(1)$ from $f_{t,i}$
    and redefine \eqref{eq:JCST2022:OMD_second_updating_4} as follows
    \begin{equation}
    \label{eq:JCST2022:OMD_second_updating_6}
        f_{t+1,i}=\mathop{\arg\min}_{f\in\mathbb{H}_i}
        \left\{\left\langle f,\tilde{\nabla}_{t,i}\right\rangle
        +\mathcal{B}_{\psi_{i}}\left(f,\bar{f}_{t,i}(2)\right)\right\}.
    \end{equation}

    Let ${\bm p}_1$ be the uniform distribution
    and ${\bm p}_{t+1}$ follow \eqref{eq:JCST2022:updating_sampling_probability},
    in which $c_{t,i}$ follows \eqref{eq:JCST2022:smooth_loss:unsigned_criterion}.
    \begin{definition}
        At each round $t$,
        $\forall \kappa_i\in\mathcal{K}$,
        let
        \begin{equation}
        \label{eq:JCST2022:smooth_loss:unsigned_criterion}
            c_{t,i}=
            \left\{
            \begin{array}{ll}
                \ell'(f_{t}({\bm x}_t),y_t)\cdot
            \left(f_{t,i}({\bm x}_t)-\min_{j\in[K]}f_{t,j}({\bm x}_t)\right),&
            ~\mathrm{if}~\ell'(f_t)>0,\\
                \ell'(f_{t}({\bm x}_t),y_t)\cdot
            \left(f_{t,i}({\bm x}_t)-\max_{j\in[K]}f_{t,j}({\bm x}_t)\right),&~\mathrm{otherwise}.
            \end{array}
            \right.
        \end{equation}
    \end{definition}

    We name this algorithm M-OMD-S (Memory bounded OMD for Smooth loss function)
    and present the pseudo-code
    in Algorithm \ref{alg:JCST2022:M-OMD-S}.

    \begin{algorithm}[H]
        \caption{\small{M-OMD-S}}
        \footnotesize
        \label{alg:JCST2022:M-OMD-S}
        \begin{algorithmic}[1]
        \REQUIRE{$\{\lambda_i\}^K_{i=1}$, $\{\eta_t\}^T_{t=1}$, $\mathcal{R}$, $U$.}
        \ENSURE{$S=\emptyset$, $f_{1,i}=0, i\in[K]$.}
        \FOR{$t=1,2,\ldots,T$}
            \STATE Receive ${\bm x}_t$
            \FOR{$i=1,\ldots,K$}
                \STATE Find ${\bm x}_{i(s_t)}$
                \STATE Compute $f_{t,i}({\bm x}_t)$
            \ENDFOR
            \STATE Output $f_t({\bm x}_t)$ or $\hat{y}_t=\mathrm{sign}(f_t({\bm x}_t))$
            and receive $y_t$
            \IF{$\mathrm{con} (a)$}
                \FOR{$i = 1,\ldots,K$}
                    \STATE
                    Compute $f_{t+1,i}$ following \eqref{eq:JCST2022:smooth_loss:surrogate_mirror_descent}
                \ENDFOR
            \ELSE
                \STATE Sample $b_{t}$
                \IF{$b_t=1$ and $\vert S\vert\leq \alpha\mathcal{R}-1$}
                    \STATE $\forall i\in[K]$,
                update $f_{t+1,i}$ following \eqref{eq:JCST2022:smooth_loss:mirror_descent}
                \ENDIF
                \IF{$b_t=1$~$\mathrm{and}$~$\vert S\vert= \alpha\mathcal{R}$}
                    \STATE Remove the first $\frac{1}{2}\alpha\mathcal{R}$ support vectors from $S$
                    \STATE $\forall i\in[K]$,
                    update $f_{t+1,i}$ following \eqref{eq:JCST2022:OMD_second_updating_6}
                \ENDIF
                    \STATE Update $S=S\cup\{({\bm x}_t,y_t):b_t=1\}$
            \ENDIF
                \STATE Compute ${\bm p}_{t+1}$ by \eqref{eq:JCST2022:updating_sampling_probability}
        \ENDFOR
        \end{algorithmic}
    \end{algorithm}

    \begin{lemma}
    \label{lem:JCST2022:number_epoches_smooth_loss_functions}
        Suppose $\ell(\cdot,\cdot)$ satisfies
        Assumption \ref{ass:JCST2022:property_smooth_loss}.
        Let $\delta\in(0,1)$ and $\alpha\mathcal{R}:=B\geq21\ln\frac{1}{\delta}$.
        For any $\mathcal{I}_T$,
        with probability at least $1-\frac{T}{B}\Theta(\lceil\ln{T}\rceil)\delta$,
        the times that M-OMD-S executes removing operation on $S$ are
        $\left\lceil\frac{4G_2\hat{L}_{1:T}}{(B-\frac{4}{3}\ln\frac{1}{\delta})G_1}\right\rceil$ at most,
        in which
        $$
            \hat{L}_{1:T}=\sum^T_{t=1}\ell(f_t({\bm x}_t),y_t).
        $$
    \end{lemma}

    It is important that
    the times of removing operation depend on the cumulative losses of the algorithm,
    rather than $T$.

    \begin{theorem}
    \label{thm:JCST2022:small_loss_bound:M-OMD-S}
        Suppose $\mathcal{R}$ and $\ell(\cdot,\cdot)$ satisfy
        Lemma \ref{lem:JCST2022:number_epoches_smooth_loss_functions}.
        Let $\delta\in(0,1)$, $K\leq d$, $\eta_{t}$ follow \eqref{eq:JCST2022:updating_sampling_probability},
        \begin{align}
        \gamma_{t}
        =&\frac{\sqrt{2\ln{K}}}{\sqrt{1+\sum^{t}_{\tau=1}\vert\ell'(f_\tau({\bm x}_\tau),y_\tau)\vert}},
        \label{eq:JCST2022:learning_rate:M-OMD-S}\\
        U\leq&\frac{1}{8G_2}\sqrt{B-\frac{4}{3}\ln\frac{1}{\delta}},\nonumber
        \end{align}
        and $\lambda_i=\frac{2U}{G_1\sqrt{B}}$ for all $i\in[K]$.
        With probability at least $1-\frac{3T}{\alpha\mathcal{R}}\Theta(\lceil\ln{T}\rceil)\delta$,
        the regret of M-OMD-S satisfies
        \begin{align*}
            \forall \kappa_i\in\mathcal{K},\forall f\in \mathbb{H}_i,&\quad
            \mathrm{Reg}(f)\leq\\
            &\frac{24UG_2L_{T}(f)}{\sqrt{B-\frac{4}{3}\ln\frac{1}{\delta}}}
            +UG_1\sqrt{B}
            +\frac{100U^2G_2G_1\ln\frac{1}{\delta}}{(1-\gamma)^2}+
            \frac{10U}{(1-\gamma)^{\frac{3}{2}}}\sqrt{G_2G_1\ln\frac{1}{\delta}}
            \sqrt{L_T(f)+\frac{UG_1\sqrt{B}}{4}},
        \end{align*}
        where $\gamma=\frac{6UG_2}{\sqrt{B-\frac{4}{3}\ln\frac{1}{\delta}}}$.
    \end{theorem}

    In the case of $B\ll T$,
    the dominant term in the upper bound is $O\left(\frac{L_T(f)}{\sqrt{B}}\right)$.
    Previous work proved a $\Omega\left(\frac{T}{B}\right)$ lower bound \cite{Zhang2013Online}.
    Thus our result proves a new trade-off between memory constraint and regret
    which gives an answer to \textbf{Q}1.
    If $\min_{f\in\mathbb{H}_{\kappa^\ast}}L_{T}(f)=o(T)$,
    then a $\Theta(\ln{T})$ memory budget is enough to
    achieve a sub-linear regret which answers \textbf{Q}2.
    By Definition \ref{JCST:def:learnability},
    we also conclude that $\mathbb{H}_{\kappa^\ast}$ is online learnable.
    The penalty term for not knowing the optimal kernel is $O\left(U\sqrt{L_T(f)\ln{K}}\right)$
    which only depends on $O\left(\sqrt{\ln{K}}\right)$.
    In this case,
    online kernel selection is not much harder than online kernel learning.

    Theorem \ref{thm:JCST2022:small_loss_bound:M-OMD-S}
    can also recover some state-of-the-art results.
    Let $B=\Theta(T)$.
    We obtain a $O\left(U\sqrt{T\ln{K}}\right)$ regret bound
    which is same with the regret bound in \cite{Sahoo2014Online,Foster2017Parameter,Shen2019Random}.
    If $K=1$ and $B=\Theta(L_T(f))$,
    then we obtain a $O\left(U\sqrt{L_T(f)}\right)$ regret bound,
    matching the regret bound in \cite{Zhang2013Online}.
    Next we prove that M-OMD-S is optimal.
    \begin{theorem}[Lower Bound]
    \label{thm:JCST2022:lower_bound_small_loss}
        Let $\ell(u,y)=\ln(1+\exp(-yu))$ which satisfies
        Assumption \ref{ass:JCST2022:property_smooth_loss}
        and $\kappa({\bm x},{\bm v})=
        \langle {\bm x},{\bm v}\rangle^p$
        where $p>0$ is a constant.
        Let $U<\frac{\sqrt{3B}}{5}$
        and $\mathbb{H}=\{f\in\mathcal{H}_\kappa:\Vert f\Vert_{\mathcal{H}}\leq U\}$.
        There exists a $\mathcal{I}_T$
        such that the regret of any online kernel algorithms storing $B$ examples satisfies
        $$
            \exists f\in\mathbb{H},\quad\mathbb{E}\left[\mathrm{Reg}(f)\right]=\Omega\left(U\cdot\frac{L_T(f)}{\sqrt{B}}\right).
        $$
        The expectation is w.r.t. the randomness of algorithm and environment.
    \end{theorem}

    Our lower bound improves the previous
    lower bound $\Omega\left(\frac{T}{B}\right)$ \cite{Zhang2013Online}.
    Next we show an algorithm-dependent upper bound.
    Let
    \begin{align*}
        \mathcal{J}=&\left\{t\in [T]:\vert S\vert=\alpha\mathcal{R}, b_t=1\right\},\\
        \Lambda_i=&\sum_{t\in \mathcal{J}}
        \left[\left\Vert f_{t,i}(2)-f\right\Vert^2_{\mathcal{H}_i}
        -\left\Vert f_{t,i}-f\right\Vert^2_{\mathcal{H}_i}\right].
    \end{align*}
    There must be a constant $\xi_i\in(0,4]$
    such that $\Lambda_i=\xi_i U^2\vert \mathcal{J}\vert$.
    It is natural that if $(f_{t,i}(2)-f)$ is close to $(f_{t,i}-f)$,
    then $\xi_i\ll 4$.
    The restart technique sets $f_{t,i}(2)=0$ implying
    $\Lambda_i\leq U^2\vert \mathcal{J}\vert$.

    \begin{theorem}[Algorithm-dependent Bound]
    \label{thm:JCST2022:algorithm_dependent:M-OMD-S}
        Suppose the conditions in Theorem \ref{thm:JCST2022:small_loss_bound:M-OMD-S}
        are satisfied.
        Let $\delta\in(0,1)$ and $U\leq\frac{B-\frac{4}{3}\ln\frac{1}{\delta}}{32G_2}$.
        If
        $\xi_i<\frac{1}{\vert\mathcal{J}\vert}$ for all $i\in[K]$,
        then with probability at least $1-\frac{3T}{\alpha\mathcal{R}}\Theta(\lceil\ln{T}\rceil)\delta$,
        the regret of M-OMD-S satisfies,
        \begin{align*}
            \forall \kappa_i\in\mathcal{K},\forall f\in \mathbb{H}_i,\quad\mathrm{Reg}(f)\leq
            \frac{32UG_2L_{T}(f)}{B-\frac{4}{3}\ln\frac{1}{\delta}}
            +
            \frac{144U^2G_2G_1\ln\frac{1}{\delta}}{(1-\gamma)^2}+
            \frac{12U}{(1-\gamma)^{\frac{3}{2}}}\sqrt{G_2G_1L_T(f)\ln\frac{1}{\delta}},
        \end{align*}
        where $\gamma=\frac{16UG_2}{B-\frac{4}{3}\ln\frac{1}{\delta}}$.
    \end{theorem}

    The dominate term in Theorem \ref{thm:JCST2022:algorithm_dependent:M-OMD-S}
    is $O\left(\frac{UG_2L_{T}(f)}{B}\right)$,
    while the dominate term in Theorem \ref{thm:JCST2022:small_loss_bound:M-OMD-S}
    is $O\left(\frac{UG_2L_{T}(f)}{\sqrt{B}}\right)$.
    The restart technique only enjoys the regret bound given in
    Theorem \ref{thm:JCST2022:small_loss_bound:M-OMD-S}.
    Thus removing a half of examples can be better than the restart technique.

\section{Experiments}

    In this section,
    we will verify the following two goals,
    \begin{enumerate}[\textbf{G}1]
      \item Learning is possible in the optimal hypothesis space within a small memory budget.\\
       We aim to verify the two data-dependent complexities
       are smaller than the worst-case complexity $T$, that is,
      $\min_{\kappa_i,i\in[K]}\mathcal{A}_{T,\kappa_i}\ll T$ and
      $\min_{f\in\mathbb{H}_i,i\in[K]}L_T(f)\ll T$.
      \item Our algorithms perform better than baseline algorothms within a same memory budget.
    \end{enumerate}

\subsection{Experimental Setting}

    We implement M-OMD-S with the logistic loss function.
    Both M-OMD-H and M-OMD-S are suitable for online binary classification tasks.
    We adopt the Gaussian kernel
    $\kappa({\bm x},{\bm v})=\exp(-\frac{\Vert{\bm x}-{\bm v}\Vert^2_2}{2\sigma^2})$,
    and choose eight classification datasets.
    The information of the datasets is shown in Table \ref{tab:JCST2022:Datasets}.
    The w8a, the ijcnn1, and the cod-rna datasets are download from the LIBSVM website
    \footnote{\url{https://www.csie.ntu.edu.tw/~cjlin/libsvmtools/datasets/}, Jul. 2024.}.
    The other datasets are downloaded from UCI machine learning repository
    \footnote{\url{http://archive.ics.uci.edu/ml/datasets.php}, Jul. 2024.}.
    All algorithms are implemented in R on a Windows machine with
    2.8 GHz Core(TM) i7-1165G7 CPU.
    We execute each experiment ten times with random permutation of all datasets.

    \begin{table*}[!htp]
      \centering
      \renewcommand\arraystretch{1.3}
      \caption{Datasets Used in Experiments}
      \label{tab:JCST2022:Datasets}
      {\footnotesize
      \begin{tabular}{lrrlrrlrr}
        \hline
        Datasets    &\#sample &\# Feature &Datasets    &\#sample &\# Feature &Datasets    &\#sample &\# Feature \\
        \hline
        mushrooms   & 8,124   & 112 & phishing & 11,055   & 68  & magic04   & 19,020 & 10  \\
        a9a         & 48,842  & 123 & w8a      & 49,749   & 300 & SUSY      & 50,000 & 18  \\
        ijcnn1      & 141,691 & 22  & cod-rna  &271,617   & 8   & \multicolumn{3}{c}{-}   \\
        \hline
      \end{tabular}
      }
    \end{table*}

    We compare our algorithms with two baseline algorithms.
    \begin{itemize}
      \item LKMBooks: OKS with memory constraint via random adding examples \cite{Li2022Worst}.
      \item Raker: online multi-kernel learning using random features \cite{Shen2019Random}.
    \end{itemize}
    Our algorithms improve the regret bounds of Raker and LKMBooks,
    and we expect that our algorithms also perform better.
    Although there are other online kernel selection algorithm,
    such as OKS \cite{Yang2012Online},
    and online multi-kernel learning algorithms,
    such as OMKR \cite{Sahoo2014Online} and OMKC \cite{Hoi2013Online},
    we do not compare with these algorithms,
    since they do not limit the memory budget.

    We select $K=5$ Gaussian kernel function by setting $\sigma=2^{[-2:2:6]}$.
    For the baseline algorithms,
    we tune the hyper-parameters following original papers.
    To be specific,
    for LKMBooks,
    there are three hyper-parameters $\nu\in(0,1)$, $\lambda$ and $\eta$.
    Both $\lambda$ and $\eta$ are learning rates.
    We reset $\lambda=\frac{10}{\sqrt{1-\nu}T}\sqrt{(1+\nu)B}$ for improving performance
    and tune $\nu\in(\frac{1}{2},\frac{1}{3},\frac{1}{4})$.
    For Raker,
    there is a learning rate $\eta$ and a regularization parameter $\lambda$.
    We tune $\eta=\frac{10^{-3:1:3}}{\sqrt{T}}$ and $\lambda \in\{0.05,0.005,0.0005\}$.
    For M-OMD-H and M-OMD-S,
    we can set an aggressive value of $U=\sqrt{B}$
    following Theorem \ref{thm:JCST2022:algorithm_dependent:M-OMD-H}
    and Theorem \ref{thm:JCST2022:algorithm_dependent:M-OMD-S}.
    We tune the learning rate $\lambda_i=c\frac{U}{\sqrt{B}}$ where $c\in\{2,1,0.5\}$.
    The value of $\eta_t$ follows \eqref{eq:JCST2022:updating_sampling_probability}.
    The value of $\gamma_t$ follows \eqref{eq:JCST2022:learning_rate}
    and \eqref{eq:JCST2022:learning_rate:M-OMD-S}.
    For M-OMD-H, we set $M=10$.

    We always set $D=400$ in Raker where $D$ is the number of random features adopted by Raker
    and plays a similar role with $B$.
    For LKMBooks and M-OMD-S,
    we always set the budget $B=400$.
    Note that M-OMD-H uniformly allocates the memory budget over $K$ hypotheses and the reservoir.
    Thus we separately set $B=100$ and $B=400$ for M-OMD-H,
    denoted as M-OMD-H $(B=100)$ and M-OMD-H $(B=400)$.
    Raker, LKMBooks, M-OMD-S, and M-OMD-H $(B=100)$ use a same memory budget,
    while M-OMD-H $(B=400)$ uses a larger memory budget.

\subsection{Experimental Results}

    We separately use the hinge loss function and the logistic loss function.

\subsubsection{Hinge Loss Function}

    We first compare the AMR of all algorithms.
    Table \ref{tab:JCST2022:experimental_results:hinge_loss} shows the results.
    As a whole,
    M-OMD-H $(B=400)$ performs best on most of datasets,
    while M-OMD-H $(B=100)$ performs worst.
    The reason is that M-OMD-H uniformly allocates the memory budget over all kernels
    thereby inducing a $O\left(\sqrt{K}\right)$ factor on the regret bound.
    M-OMD-H $(B=400)$ uses a larger memory budget, thereby performing better.
    The experimental results do not completely coincide with \textbf{G}2.
    It is left for future work to study whether we can avoid allocating the memory budget.

    Table \ref{tab:JCST2022:experimental_results:hinge_loss} also
    reports the average running time.
    It is natural that M-OMD-H $(B=400)$
    requires the longest running time,
    since it uses a larger memory budget.
    The running time of the other three algorithms is similar.

    Next we analyze the value of $\mathcal{A}_{T,\kappa^{\ast}}$.
    Computing $\mathcal{A}_{T,\kappa^{\ast}}$ requires $O(T^2)$ time.
    For simplicity,
    we will define a proxy of $\mathcal{A}_{T,\kappa^{\ast}}$.
    Let
    $$
        \hat{\mathcal{A}}_{T,i}:=\sum^T_{\tau=1} \left\Vert \nabla_{\tau,i}-
            \hat{\nabla}_{\tau,i}\right\Vert^2_{\mathcal{H}_{i}}\cdot\mathbb{I}_{\nabla_{\tau,i}\neq 0}.
    $$
    To be specific,
    we use $\min_{i\in[K]}\hat{\mathcal{A}}_{T,i}$ as a proxy of $\mathcal{A}_{T,\kappa^{\ast}}$.
    In fact,
    our regret bound depends on $\hat{\mathcal{A}}_{T,i}$
    which satisfies $\hat{\mathcal{A}}_{T,i}=\tilde{O}(\mathcal{A}_{T,\kappa_i})$.
    To obtain a precise estimation of $\mathcal{A}_{T,\kappa_i}$,
    we again run M-OMD-H with $M=30$ and $B=400$.

    Table \ref{tab:JCST2022:experimental_results:hinge_loss}
    shows the results.
    We find that $\mathcal{A}_{T,\kappa^{\ast}}<T$ on all datasets.
    Especially,
    $\mathcal{A}_{T,\kappa^{\ast}}\ll T$ on the mushrooms, the phishing, and the w8a datasets.
    In the optimal hypothesis space,
    a small memory budget is enough to guarantee a small regret
    and thus learning is possible.
    The experimental results verify \textbf{G}1.

    \begin{table*}[!htb]
      \centering
      \renewcommand\arraystretch{1.3}
      \caption{
      Experimental Results Using the Hinge Loss Function}
      \label{tab:JCST2022:experimental_results:hinge_loss}
      {\footnotesize
      \begin{tabular}{lrrrrrrrr}
        \hline
        \multirow{2}{*}{Algorithm}&\multicolumn{4}{c}{mushrooms, $T=8124$} &\multicolumn{4}{c}{phishing, $T=11055$}\\
        \cline{2-9}&AMR &$B\vert D$ &Time (s) & $\mathcal{A}_{T,\kappa^{\ast}}$
                   &AMR &$B\vert D$ &Time (s) & $\mathcal{A}_{T,\kappa^{\ast}}$  \\
        \hline
        Raker \cite{Shen2019Random}
                       & 11.97 $\pm$ 0.93  & 400   & 1.57 & - & 9.13  $\pm$ 0.28  & 400   & 2.06 & - \\
        LKMBooks \cite{Li2022Worst}
                       & 6.04  $\pm$ 0.62  & 400   & 1.64 & - & 12.81 $\pm$ 0.54  & 400   & 2.21 & - \\
        M-OMD-H        & 8.09  $\pm$ 2.51  & 100   & 5.62 & - & 19.41 $\pm$ 3.78  & 100   & 5.05 & - \\
        M-OMD-H        & \textbf{0.69}  $\pm$ \textbf{0.07}  & 400   & 10.97 & 220
                       & \textbf{8.87}  $\pm$ \textbf{0.40}  & 400   & 10.74 & 2849 \\
        \hline
        \multirow{2}{*}{Algorithm}&\multicolumn{4}{c}{magic04, $T=19020$} &\multicolumn{4}{c}{a9a, $T=48842$}\\
        \cline{2-9}&AMR &$B\vert D$ &Time (s) & $\mathcal{A}_{T,\kappa^{\ast}}$
                   &AMR &$B\vert D$ &Time (s) & $\mathcal{A}_{T,\kappa^{\ast}}$  \\
        \hline
        Raker \cite{Shen2019Random}
                     & 32.14 $\pm$ 0.43  & 400   & 3.25   & - & 23.47 $\pm$ 0.23  & 400   & 9.45  & - \\
        LKMBooks \cite{Li2022Worst}
                     & 23.66 $\pm$ 0.51  & 400   & 1.31   & - & 21.29 $\pm$ 1.36  & 400   & 10.15 & - \\
        M-OMD-H        & 29.08 $\pm$ 3.78  & 100   & 5.50   & - & 24.84 $\pm$ 1.43  & 100   & 31.39 & - \\
        M-OMD-H        & \textbf{22.75} $\pm$ \textbf{0.57}  & 400   & 6.09  & 10422
                       & \textbf{19.88} $\pm$ \textbf{1.44}  & 400   & 49.29 & 19393 \\
        \hline
        \multirow{2}{*}{Algorithm}&\multicolumn{4}{c}{w8a, $T=49749$} &\multicolumn{4}{c}{SUSY, $T=50000$}\\
        \cline{2-9}&AMR &$B\vert D$ &Time (s) & $\mathcal{A}_{T,\kappa^{\ast}}$
                   &AMR &$B\vert D$ &Time (s) & $\mathcal{A}_{T,\kappa^{\ast}}$  \\
        \hline
        Raker \cite{Shen2019Random}
                     & 2.98 $\pm$ 0.00  & 400   & 11.62 & - & 29.62 $\pm$ 0.46  & 400   & 8.88 & - \\
        LKMBooks \cite{Li2022Worst}
                     & 2.97 $\pm$ 0.00  & 400   & 23.52 & - & 29.37 $\pm$ 0.87  & 400   & 5.47 & - \\
        M-OMD-H        & 2.99 $\pm$ 0.10  & 100   & 43.20 & - & 36.21 $\pm$ 2.08  & 100   & 16.29 & - \\
        M-OMD-H        & \textbf{2.65}  $\pm$ \textbf{0.12}  & 400   & 84.80 & 4200
                       & \textbf{27.76} $\pm$ \textbf{1.04}  & 400   & 29.45 & 28438 \\
        \hline
        \multirow{2}{*}{Algorithm}&\multicolumn{4}{c}{ijcnn1, $T=141691$} &\multicolumn{4}{c}{con-rna, $T=271617$}\\
        \cline{2-9}&AMR &$B\vert D$ &Time (s) & $\mathcal{A}_{T,\kappa^{\ast}}$
                   &AMR &$B\vert D$ &Time (s) & $\mathcal{A}_{T,\kappa^{\ast}}$  \\
        \hline
        Raker \cite{Shen2019Random}
        & 9.49 $\pm$ 0.08   & 400   & 23.54 & - & \textbf{11.90} $\pm$ \textbf{0.31}  & 400   & 53.32  & - \\
        LKMBooks \cite{Li2022Worst}
                       & 9.58 $\pm$ 0.01   & 400   & 18.85 & - & 12.57 $\pm$ 0.26  & 400   & 19.17  & - \\
        M-OMD-H        & 9.98 $\pm$ 0.29   & 100   & 33.15 & - & 14.59 $\pm$ 1.47  & 100   & 70.33  & - \\
        M-OMD-H        & 9.57 $\pm$ 0.43   & 400   & 54.81 & 33626 & 12.46 $\pm$ 0.54  & 400   & 130.09 & 62259 \\
        \hline
      \end{tabular}
      }
    \end{table*}

\subsubsection{Logistic Loss Function}

    We first compare the AMR of all algorithms.
    Table \ref{tab:JCST2022:experimental_results:smooth_loss} reports the results.
    As a whole,
    M-OMD-S performs best on most of datasets.
    On the phishing, the SUSY, and the cod-rna datasets,
    Raker performs better than M-OMD-S.
    We have found that the current regret analysis of Raker is not tight
    and can be improved by utilizing
    Assumption \ref{ass:JCST2022:property_smooth_loss}.
    Thus it is reasonable that Raker performs similar with M-OMD-S.
    Besides,
    M-OMD-S performs better than LKMBooks on all datasets.
    The experimental results coincide with the theoretical analysis
    that is, M-OMD-S enjoys a better regret bound than LKMBooks.
    The experimental results verify \textbf{G}2.
    It is natural that the running time of all algorithms is similar
    as they enjoy the same time complexity.

    Finally, we analyze the value of $L^\ast$.
    For simplicity,
    we use the cumulative losses of M-OMD-S, $\hat{L}_{1:T}$, as a proxy of $L^\ast$.
    Note that $L^\ast \leq \hat{L}_{1:T}$.
    Table \ref{tab:JCST2022:experimental_results:smooth_loss}
    shows that $L^\ast<T$ on all datasets.
    Besides,
    on the mushrooms dataset, phishing dataset, and w8a dataset,
    $L^\ast \ll T$.
    In the optimal hypothesis space,
    a small memory budget is enough to guarantee a small regret
    and thus learning is possible.
    The experimental results verify \textbf{G}1.

    \begin{table*}[!htp]
      \centering
      \renewcommand\arraystretch{1.3}
      \caption{
      Experimental Results Using the Logistic Loss Function}
      \label{tab:JCST2022:experimental_results:smooth_loss}
      {\footnotesize
      \begin{tabular}{lrrrrrrrr}
        \hline
        \multirow{2}{*}{Algorithm}&\multicolumn{4}{c}{mushrooms, $T=8124$} &\multicolumn{4}{c}{phishing, $T=11055$}\\
        \cline{2-9}&AMR &$B\vert D$       &Time (s) & $L^\ast$ &AMR &$B\vert D$       &Time (s) & $L^\ast$ \\
        \hline
        Raker \cite{Shen2019Random}
        & 11.47 $\pm$ 0.75  & 400   & 1.64  & - & \textbf{9.07}  $\pm$ \textbf{0.44}  & 400   & 2.05   & -\\
        LKMBooks \cite{Li2022Worst}
            & 10.01 $\pm$ 1.44  & 400   & 2.12  & - & 15.86 $\pm$ 1.28  & 400   & 1.94  & -\\
        M-OMD-S        & \textbf{2.70}  $\pm$ \textbf{0.32}  & 400   & 1.74   & 160
                       & 10.38 $\pm$ 0.31  & 400   & 3.55   & 3012\\
        \hline
        \multirow{2}{*}{Algorithm}&\multicolumn{4}{c}{magic04, $T=19020$} &\multicolumn{4}{c}{a9a, $T=48842$}\\
        \cline{2-9}&AMR &$B\vert D$       &Time (s) & $L^\ast$ &AMR &$B\vert D$       &Time (s)  & $L^\ast$ \\
        \hline
        Raker \cite{Shen2019Random}
                     & 30.74 $\pm$ 1.11  & 400   & 3.56   & - & 23.13 $\pm$ 0.17  & 400   & 9.90   & -\\
        LKMBooks \cite{Li2022Worst}
                     & 25.72 $\pm$ 0.64  & 400   & 1.30   & - & 22.38 $\pm$ 1.29  & 400   & 10.98  & -\\
        M-OMD-S        & \textbf{23.97} $\pm$ \textbf{0.23}  & 400   & 1.81    & 10160
                       & \textbf{19.43} $\pm$ \textbf{0.16}  & 400   & 10.41   & 20194\\
        \hline
        \multirow{2}{*}{Algorithm}&\multicolumn{4}{c}{w8a, $T=49749$} &\multicolumn{4}{c}{SUSY, $T=50000$}\\
        \cline{2-9}&AMR &$B\vert D$       &Time (s) & $L^\ast$ &AMR &$B\vert D$       &Time (s)  & $L^\ast$ \\
        \hline
        Raker \cite{Shen2019Random}
                & 2.97  $\pm$ 0.00  & 400   & 12.06 & - & \textbf{28.62} $\pm$ \textbf{0.44}  & 400   & 9.02  & -\\
        LKMBooks \cite{Li2022Worst}
                & 2.98  $\pm$ 0.01  & 400   & 6.10  & - & 30.72 $\pm$ 1.33  & 400   & 6.45  & 0\\
        M-OMD-S        & \textbf{2.83} $\pm$ \textbf{0.03}  & 400   & 21.44   & 5116
                       & 29.03 $\pm$ 0.16  & 400   & 8.24  & 28582\\
        \hline
        \multirow{2}{*}{Algorithm}&\multicolumn{4}{c}{ijcnn1, $T=141691$} &\multicolumn{4}{c}{con-rna, $T=271617$}\\
        \cline{2-9}&AMR &$B\vert D$ &Time (s)  &$L^\ast$ &AMR &$B\vert D$   &Time (s)  & $L^\ast$ \\
        \hline
        Raker \cite{Shen2019Random}
              & 9.53  $\pm$ 0.03  & 400   & 24.94 & - & \textbf{12.27} $\pm$ \textbf{0.57}  & 400   & 52.24  & -\\
        LKMBooks \cite{Li2022Worst}
                       & 9.57  $\pm$ 0.00  & 400   & 16.95 & - & 13.91 $\pm$ 1.31  & 400   & 25.55  & -\\
        M-OMD-S        & \textbf{9.42} $\pm$ \textbf{0.05}  & 400   & 19.33   & 39208
                       & 13.09 $\pm$ 0.04 & 400   & 31.20   & 94326\\
        \hline
      \end{tabular}
      }
    \end{table*}

\section{Conclusion}

    Learnability is an essential problem in online kernel selection with memory constraint.
    Our work gave positive results on learnability through data-dependent regret analysis,
    in contrast to previous negative results obtained from worst-case regret analysis.
    We characterized the regret bounds via the kernel alignment and the cumulative losses,
    and gave new trade-offs between regret and memory constraint.
    If there is a kernel function matches well with the data,
    i.e., the kernel alignment and the cumulative losses is sub-linear,
    then sub-linear regret can be achieved within a $\Theta(\ln{T})$ memory constraint
    and the corresponding hypothesis space is learnable.
    Data-dependent regret analysis provides a new perspective
    for studying the learnability in online kernel selection with memory constraint.

    Deploying machine learning models on computational devices
    with limited computational resources is necessary for practical applications.
    Our algorithms limit the memory budget to $O(\ln{T})$,
    and can naturally be deployed on computational devices with limited memory.
    Even if the deployed machine learning models are not kernel classifiers,
    our work can guide how to allocate the memory resources properly.
    For instance,
    (i)
    the optimal value of memory budget depends on the hardness of problems;
    (ii) our second algorithm provides the idea of memory sharing.

    There are some important work for future study.
    Firstly,
    it is interesting to study other data complexities beyond kernel alignment and small-loss,
    such as the effective dimension of the kernel matrix.
    Secondly,
    it is important to investigate whether
    we can eliminate the $O\left(\sqrt{K}\right)$ factor in the regret bound of M-OMD-H.

%\section*{Acknowledgement}
%
%This is a preprint of an article published in Journal of Computer Science and Technology.
%The final authenticated version is available online at: https://doi.org/[insert DOI]

\bibliographystyle{IEEEtran}
\bibliography{JCST2022}

\newpage
\appendix
\setcounter{equation}{0}
\renewcommand\theequation{A\arabic{equation}}
\setcounter{theorem}{0}
\setcounter{lemma}{0}
\renewcommand\thesection{A.\arabic{section}}
\newtheorem{lemmaAppend}{Lemma}[section]

\section*{Appendix}

 In this appendix,
 we give the detailed proofs of the main theorems and lemmas.

\section{Proof of Lemma \ref{lem:JCST2022:number_of_epoch:hinge_loss}}
\label{sec:JCST2022:proof_Lemma_1}

    \begin{proof}
        We just analyze $S_i$ for a fixed $i\in[K]$.
        Let the times of removing operation be $J$.
        Denote by $B=\alpha\mathcal{R}$,
        $\mathcal{J}=\{t_r,r\in[J]\}$,
        $T_r=\{t_{r-1}+1,\ldots,t_{r}\}$ and $t_0=0$.
        For any $t\in T_r$,
        if $\nabla_{t,i}\neq 0$, $\neg \mathrm{con}(a(i))$
        and $b_{t,i}=1$,
        then $({\bm x}_t,y_t)$ will be added into $S_i$.
        For simplicity, we define a new notation $\nu_{t,i}$ as follows,
        $$
            \nu_{t,i}=\mathbb{I}_{y_tf_{t,i}({\bm x}_t)<1}\cdot\mathbb{I}_{\neg\mathrm{con}(a(i))}
        \cdot b_{t,i}.
        $$
        At the end of the $t_r$-th round,
        the following equation can be derived,
        $$
            \vert S_{i}\vert
            =\vert S_{i}(t_{r-1}+1)\vert+\sum^{t_{r}}_{t=t_{r-1}+1}\nu_{t,i}= \frac{B}{K},
        $$
        where $\vert S_{i}(t_{r-1}+1)\vert$ is defined the initial size of $S_i$.

        Let $s_r=t_{r-1}+1$.
        Assuming that there is no budget.
        We will present an expected bound on $\sum^{\bar{t}}_{t=s_r}\nu_{t,i}$
        for any $\bar{t}>s_r$.
        In the first epoch, $s_1=1$ and $\vert S_i(s_1)\vert=0$.
        Taking expectation w.r.t. $b_{t,i}$ gives
        \begin{align*}
            \mathbb{E}\left[\sum^{\bar{t}}_{t=s_1}\nu_{t,i}\right]
            =&\sum^{\bar{t}}_{t=s_1}
            \frac{\Vert \nabla_{t,i}-\hat{\nabla}_{t,i}\Vert^2_{\mathcal{H}_i}\cdot\mathbb{I}_{\nabla_{t,i}\neq 0}}
                    {\Vert \nabla_{t,i}-\hat{\nabla}_{t,i}\Vert^2_{\mathcal{H}_i}+
                    \Vert\hat{\nabla}_{t,i}\Vert^2_{\mathcal{H}_i}}\\
            \leq&\frac{2}{k_1}\underbrace{\left(1+\sum^{\bar{t}}_{t=2}\left\Vert y_t\kappa_{i}({\bm x}_t,\cdot)
            -\frac{\sum_{({\bm x},y)\in V_t}y\kappa_{i}({\bm x},\cdot)}{\vert V_t\vert}\right\Vert^2_{\mathcal{H}_i}
            \right)}
            _{\tilde{\mathcal{A}}_{[s_1,\bar{t}],\kappa_i}}\\
            =&\frac{2}{k_1}\tilde{\mathcal{A}}_{[s_1,\bar{t}],\kappa_i},
        \end{align*}
        where we use the fact $\kappa_i({\bm x}_t,{\bm x}_t)\geq k_1$.
        Let $t_1$ be the minimal $\bar{t}$ such that
        \begin{equation}
        \label{eq:JCST2022:analysis_of_number_of_epoch:first_epoch:hinge_loss}
            \frac{2}{k_1}\tilde{\mathcal{A}}_{[s_1,t_1],\kappa_i}\geq \frac{B}{K}.
        \end{equation}
        The first epoch will end at $t_1$ in expectation.
        We define $\tilde{\mathcal{A}}_{T_1,\kappa_i}:=\tilde{\mathcal{A}}_{[s_1,t_1],\kappa_i}$.

        Next we consider $r\geq 2$.
        It must be $\vert S_i(s_r)\vert=\frac{B}{2K}$.
        Similar to $r=1$,
        we can obtain
        $$
            \mathbb{E}\left[\sum^{\bar{t}}_{t=s_r}\nu_{t,i}\right]
            \leq\frac{2}{k_1}\underbrace{\sum^{\bar{t}}_{t=s_r}\left\Vert y_t\kappa_{i}({\bm x}_t,\cdot)
            -\frac{\sum_{({\bm x},y)\in V_t}y\kappa_{i}({\bm x},\cdot)}{\vert V_t\vert}\right\Vert^2_{\mathcal{H}_i}}
            _{\tilde{\mathcal{A}}_{[s_r,\bar{t}],\kappa_i}}
            =\frac{2}{k_1}\tilde{\mathcal{A}}_{[s_r,\bar{t}],\kappa_i}.
        $$
        Let $t_r$ be the minimal $\bar{t}$ such that
        \begin{equation}
        \label{eq:JCST2022:analysis_of_number_of_epoch:any_epoch:hinge_loss}
            \frac{2}{k_1}\tilde{\mathcal{A}}_{[s_r,\bar{t}],\kappa_i} \geq \frac{B}{2K},
        \end{equation}
        Let $\tilde{\mathcal{A}}_{T_r,\kappa_i}=\tilde{\mathcal{A}}_{[s_r,\bar{t}],\kappa_i}$.
        Combining
        \eqref{eq:JCST2022:analysis_of_number_of_epoch:first_epoch:hinge_loss}
        and \eqref{eq:JCST2022:analysis_of_number_of_epoch:any_epoch:hinge_loss},
        and summing over $r=1,\ldots,J$ yields
        \begin{align*}
            \frac{B}{K}+\frac{B(J-1)}{2K}
            \leq&
            \frac{2}{k_1}\tilde{\mathcal{A}}_{T_1,\kappa_i}+\sum^{J}_{r=2}\frac{2}{k_1}\tilde{\mathcal{A}}_{T_r,\kappa_i}\\
            \leq&\frac{2}{k_1}\underbrace{\sum^{T}_{t=s_1}\left\Vert y_t\kappa_{i}({\bm x}_t,\cdot)
            -\frac{\sum_{({\bm x},y)\in V_t}y\kappa_{i}({\bm x},\cdot)}{\vert V_t\vert}\right\Vert^2_{\mathcal{H}_i}}
            _{\tilde{\mathcal{A}}_{T,\kappa_i}}\\
            \leq& \frac{2}{k_1}\tilde{\mathcal{A}}_{T,\kappa_i}.
        \end{align*}
        Arranging terms gives
        \begin{equation}
        \label{eq:JCST2022:M-OMD-H:J}
            J \leq \frac{4K\tilde{\mathcal{A}}_{T,\kappa_i}}{Bk_1}-1\leq \frac{4K\tilde{\mathcal{A}}_{T,\kappa_i}}{Bk_1}.
        \end{equation}
        Taking expectation w.r.t. the randomness of reservoir sampling gives
        $$
            \mathbb{E}[J]
            \leq \frac{4K}{Bk_1}\cdot\mathbb{E}[\tilde{\mathcal{A}}_{T,\kappa_i}]
            \leq\frac{12K}{Bk_1}\mathcal{A}_{T,\kappa_i}\cdot
            \left(1+\frac{\ln{T}}{M}\right)
            +\frac{32K}{Bk_1},
        $$
        where the last inequality comes from Lemma \ref{lem:JCST2022:reservoir_estimator}.
        Omitting the last constant term concludes the proof.
    \end{proof}
    \begin{lemmaAppend}
    \label{lem:JCST2022:reservoir_estimator}
        The reservoir sampling guarantees
        $$
            \forall i\in[K],\quad\mathbb{E}\left[\tilde{\mathcal{A}}_{T,\kappa_i}\right]
            \leq 3\mathcal{A}_{T,\kappa_i}+8+\frac{3\mathcal{A}_{T,\kappa_i}\ln{T}}{M}.
        $$
    \end{lemmaAppend}
    \begin{proof}
        Let $\mu_{t,i}=-\frac{1}{t}\sum^{t}_{\tau=1}y_{\tau}\kappa_{i}({\bm x}_{\tau},\cdot)$
        and $\tau_0 = M$.
        For $t \leq \tau_0$,
        it can be verified that
        \begin{align*}
            \tilde{\mathcal{A}}_{\tau_0,\kappa_i}
            =&1+\sum^{\tau_0}_{t=2}\left\Vert -y_t\kappa_{i}({\bm x}_t,\cdot)
            -\mu_{t-1,i}\right\Vert^2_{\mathcal{H}_i}\\
            =&1+\sum^{\tau_0}_{t=2}\left\Vert -y_t\kappa_{i}({\bm x}_t,\cdot)
            -\mu_{t,i}+\mu_{t,i}-\mu_{t-1,i}\right\Vert^2_{\mathcal{H}_i}\\
            \leq&1+2\mathcal{A}_{[2:\tau_0],\kappa_i}
            +2\sum^{\tau_0}_{t=2}\left\Vert \mu_{t,i}-\mu_{t-1,i}\right\Vert^2_{\mathcal{H}_i},
        \end{align*}
        where $\mu_{0,i}=0$.
        Let $V_t$ be the reservoir at the beginning of round $t$.
        Next we consider the case $t>\tau_0$.
        \begin{align*}
            \tilde{\mathcal{A}}_{[\tau_0:T],\kappa_i}
            =& \sum^{T}_{t=\tau_0+1}\left\Vert -y_t\kappa_{i}({\bm x}_t,\cdot)
            -\tilde{\mu}_{t-1,i}\right\Vert^2_{\mathcal{H}_i}\\
            \leq& \sum^{T}_{t=\tau_0+1}
            3\left[\left\Vert y_t\kappa_{i}({\bm x}_t,\cdot)+ \mu_{t,i}\right\Vert^2_{\mathcal{H}_i}
            +\left\Vert \mu_{t,i}-{\mu}_{t-1,i}\right\Vert^2_{\mathcal{H}_i}+\left\Vert {\mu}_{t-1,i}-\tilde{\mu}_{t-1,i}\right\Vert^2_{\mathcal{H}_i}\right]\\
            =&3\mathcal{A}_{[\tau_0:T],\kappa_i}
            +3\sum^{T}_{t=\tau_0+1}\left\Vert \mu_{t,i}-{\mu}_{t-1,i}\right\Vert^2_{\mathcal{H}_i}+
            3\sum^{T}_{t=\tau_0+1}\left\Vert {\mu}_{t-1,i}+\frac{1}{\vert V_t\vert}\sum_{({\bm x},y)\in V_t}
            y\kappa_{i}({\bm x},\cdot)\right\Vert^2_{\mathcal{H}_i}.
        \end{align*}
        Taking expectation w.r.t. the reservoir sampling yields
        \begin{align*}
            &\mathbb{E}[\tilde{\mathcal{A}}_{T,\kappa_i}]\\
            =&\tilde{\mathcal{A}}_{\tau_0,\kappa_i}+\mathbb{E}[\tilde{\mathcal{A}}_{[\tau_0:T],\kappa_i}]\\
            \leq& 1+3\mathcal{A}_{T,\kappa_i}
            +3\sum^{T}_{t=2}\left\Vert\mu_{t,i}-{\mu}_{t-1,i}\right\Vert^2_{\mathcal{H}_i}
            +3\sum^{T}_{t=\tau_0+1}\mathbb{E}
            \left[\left\Vert {\mu}_{t-1,i}+\frac{1}{\vert V_t\vert}\sum_{({\bm x},y)\in V_t}
            y\kappa_{i}({\bm x},\cdot)\right\Vert^2_{\mathcal{H}_i}\right]
            \quad\quad\quad\mathrm{by}~\mathrm{Lemma}~\ref{lem:JCST2022:variance_reservoir_estimator}\\
            \leq& 1+3\mathcal{A}_{T,\kappa_i}
            +3\sum^{T}_{t=2}\left\Vert\mu_{t,i}-{\mu}_{t-1,i}\right\Vert^2_{\mathcal{H}_i}
            +\sum^{T}_{t=\tau_0+1}\frac{3\mathcal{A}_{t-1,\kappa_i}}{(t-1)\vert V_t\vert}\\
            \leq& 1+3\mathcal{A}_{T,\kappa_i}+\sum^T_{t=2}\frac{12}{t^2}
            +\frac{3\mathcal{A}_{T,\kappa_i}\ln{T}}{M}\\
            \leq& 1+3\mathcal{A}_{T,\kappa_i}+7+\frac{3\mathcal{A}_{T,\kappa_i}\ln{T}}{M},
        \end{align*}
        where $\vert V_t\vert =M$ for all $t\geq \tau_0$.
    \end{proof}

\section{Proof of Theorem \ref{thm:JCST2022:regret-bound:M-OMD-H}}

    \begin{proof}
     By the convexity of the Hinge loss function,
     we decompose the regret as follows
        \begin{align*}
           \mathrm{Reg}(f)
           =&\sum^{T}_{t=1}\ell\left(\sum^K_{j=1}p_{t,j}f_{t,j}({\bm x}_t),y_t\right)
           -\sum^{T}_{t=1}\ell\left(f({\bm x}_t),y_t\right)\\
           \leq&\sum^{T}_{t=1}\sum^K_{j=1}p_{t,j}\ell\left(f_{t,j}({\bm x}_t),y_t\right)
           -\sum^{T}_{t=1}\ell\left(f({\bm x}_t),y_t\right)\\
            \leq&\underbrace{\sum^{T}_{t=1}\left[\sum^K_{j=1}p_{t,j}\ell(f_{t,j}({\bm x}_t),y_t)
            -\ell(f_{t,i}({\bm x}_t),y_t)\right]}_{\mathcal{T}_1}+
            \underbrace{\sum_{t\in E_{T,i}}\left[\ell(f_{t,i}({\bm x}_t),y_t)
            -\ell(f({\bm x}_t),y_t)\right]}_{\mathcal{T}_2},
        \end{align*}
        where $E_{T,i}=\{t\in[T],\nabla_{t,i}\neq 0\}$.

\subsection{Analyzing $\mathcal{T}_1$}

    The following analysis is same with the proof of Theorem 3.1 in \cite{Bubeck2012Regret}.
    Let $c_{t,i}:=\ell(f_{t,i}({\bm x}_t),y_t)$.
    The updating of probability is as follows,
    $$
        p_{t+1,i}=\frac{w_{t+1,i}}{\sum^K_{j=1}w_{t+1,j}},\quad
        w_{t+1,i}=\exp\left(-\eta_{t+1}\sum^t_{\tau=1}c_{\tau,i}\right).
    $$
    Similar to the analysis of Exp3 \cite{Bubeck2012Regret},
    we define a potential function $\Gamma_t(\eta_t)$ as follows,
    \begin{align*}
        \Gamma_t(\eta_t):=\frac{1}{\eta_t}\ln\sum^K_{i=1}p_{t,i}\exp(-\eta_tc_{t,i})
        \leq-\sum^K_{i=1}p_{t,i}c_{t,i}+\frac{1}{2}\eta_t\sum^K_{i=1}p_{t,i}c^2_{t,i},
    \end{align*}
    where we use the following two inequalities
    $$
        \ln{x}\leq x-1,\forall x >0,\quad
    \exp(-x)\leq 1-x+\frac{x^2}{2},\forall x\geq 0.
    $$
    Summing over $t\in[T]$ yields
    \begin{equation}
    \label{eq:AISTATS2022_supp:time_variant_learning_rate:upper_bound}
        \sum^T_{t=1}\Gamma_t(\eta_t)\leq -\sum^T_{t=1}\langle{\bm p}_t,{\bm c}_t\rangle
        +\sum^T_{t=1}\sum^K_{i=1}\frac{\eta_t}{2}p_{t,i}c^2_{t,i}.
    \end{equation}
    On the other hand,
    by the definition of $p_{t,i}$,
    we have
    \begin{align*}
        \Gamma_t(\eta_t)
        =&\frac{1}{\eta_t}\ln
        \frac{\sum^K_{i=1}\exp\left(-\eta_{t}\sum^{t-1}_{\tau=1}c_{\tau,i}\right)\exp(-\eta_tc_{t,i})}
        {\sum^K_{j=1}\exp\left(-\eta_{t}\sum^{t-1}_{\tau=1}c_{\tau,j}\right)}\\
        =&\frac{1}{\eta_t}\ln
        \frac{\frac{1}{K}\sum^K_{i=1}\exp\left(-\eta_{t}\sum^{t}_{\tau=1}c_{\tau,i}\right)}
        {\frac{1}{K}\sum^K_{j=1}\exp\left(-\eta_{t}\sum^{t-1}_{\tau=1}c_{\tau,j}\right)}\\
        =&\bar{\Gamma}_t(\eta_t)-\bar{\Gamma}_{t-1}(\eta_t),
    \end{align*}
    where
    $
        \bar{\Gamma}_t(\eta)
        =\frac{1}{\eta}\ln\frac{1}{K}\sum^K_{j=1}\exp\left(-\eta\sum^{t}_{\tau=1}c_{\tau,j}\right).
    $\\
    Without loss of generality,
    let $\bar{\Gamma}_0(\eta)=0$.
    Summing over $t=1,\ldots,T$ yields
    $$
        \sum^T_{t=1}\Gamma_t(\eta_t)
        =\bar{\Gamma}_T(\eta_T)-\bar{\Gamma}_0(\eta_1)
        +\sum^{T-1}_{t=1}\left[\bar{\Gamma}_t(\eta_t)-\bar{\Gamma}_{t}(\eta_{t+1})\right],
    $$
    where $\bar{\Gamma}_T(\eta_T)\geq\frac{1}{\eta_T}\ln\frac{1}{K}-\sum^{T}_{\tau=1}c_{\tau,i}$.
    Combining with the upper bound \eqref{eq:AISTATS2022_supp:time_variant_learning_rate:upper_bound},
    we obtain
    $$
        \sum^T_{t=1}\langle{\bm p}_t,{\bm c}_t\rangle-\sum^{T}_{\tau=1}c_{\tau,i}
        \leq\frac{1}{\eta_T}\ln{K}+\sum^{T-1}_{t=1}\left[\bar{\Gamma}_{t}(\eta_{t+1})-\bar{\Gamma}_t(\eta_t)\right]
        +\sum^T_{t=1}\sum^K_{i=1}\frac{\eta_t}{2}p_{t,i}c^2_{t,i}.
    $$
    For simplicity,
    let $\bar{C}_{t,j}:=\sum^{t}_{\tau=1}c_{\tau,j}$.
    The first derivative of $\bar{\Gamma}_t(\eta)$ w.r.t. $\eta$ is as follows
    \begin{align*}
        \frac{\mathrm{d}\,\bar{\Gamma}_t(\eta)}{\mathrm{d}\,\eta}
        =&\frac{-\ln
        \sum^K_{j=1}\frac{\exp\left(-\eta \bar{C}_{t,j}\right)}{K}}{\eta^2}
        -\frac{\frac{1}{K}\sum^K_{j=1}\bar{C}_{t,j}\exp\left(-\eta \bar{C}_{t,j}\right)}
        {\frac{\eta}{K}\sum^K_{j=1}\exp\left(-\eta \bar{C}_{t,j}\right)}\\
        =&\frac{1}{\eta^2}\mathrm{KL}(\tilde{p}_t,\frac{1}{K}) \\
        \geq& 0
    \end{align*}
    where $\tilde{p}_{t,j}=\frac{\exp\left(-\eta \bar{C}_{t,j}\right)}{\sum^K_{i=1}
        \exp\left(-\eta \bar{C}_{t,i}\right)}$.
    Since $\eta_{t+1}\leq \eta_t$,
    we have $\bar{\Gamma}_{t}(\eta_{t+1})\leq\bar{\Gamma}_t(\eta_t)$.
    Combining all results, we have
    \begin{align}
        &\sum^T_{t=1}\langle{\bm p}_t,{\bm c}_t\rangle-\sum^{T}_{\tau=1}c_{\tau,i}\nonumber\\
        \leq&\frac{\ln{K}}{\eta_T}-\frac{\ln{K}}{\eta_1}+\sum^T_{t=1}\sum^K_{i=1}\frac{\eta_t}{2}p_{t,i}c^2_{t,i}
        \nonumber\\
        \leq&
        \frac{\sqrt{\ln{K}}}{\sqrt{2}}\cdot
        \sqrt{1+\sum^{T-1}_{\tau=1}\langle{\bm p}_\tau,{\bm c}^2_\tau\rangle}-\frac{\sqrt{\ln{K}}}{\sqrt{2}}+
        \sqrt{\ln{K}}
        \left(\sqrt{2\sum^{T}_{\tau=1}\langle{\bm p}_\tau,{\bm c}^2_\tau\rangle}
        +\frac{\max_{t,j}c_{t,j}}{\sqrt{2}}\right)
        \quad\quad\quad\mathrm{by}~\mathrm{Lemma}~\ref{lem:JCST2022:analysis_second_moment_TV}\nonumber\\
        \lesssim& \frac{3}{\sqrt{2}}
        \sqrt{\max_{t,j}c_{t,j}\cdot\sum^{T}_{\tau=1}\langle{\bm p}_\tau,{\bm c}_\tau\rangle\ln{K}}
        \label{eq:JCST2023_supp:initial_regret_PEA}.
    \end{align}
    Solving for $\sum^T_{t=1}\langle {\bm p}_t,{\bm c}_t\rangle$ gives
    \begin{equation}
    \label{eq:JCST2022:kernel_alignment_bound:first_part_of_regret}
        \mathcal{T}_1=\sum^T_{t=1}[\langle{\bm p}_t,{\bm c}_t\rangle-c_{t,i}]
        \leq \frac{3}{\sqrt{2}}\sqrt{\max_{t,j} c_{t,j}\cdot\sum^{T}_{\tau=1}c_{\tau,i}\ln{K}}
        +\frac{9}{2}\max_{t,j} c_{t,j}\cdot\ln{K}.
    \end{equation}

\subsection{Analyzing $\mathcal{T}_2$}

    We decompose $E_{T,i}$ as follows.
        \begin{align*}
            E_i=&\{t\in E_{T,i}:\mathrm{con} (a(i))\},\\
            \mathcal{J}_i=&\{t\in E_{T,i}:\vert S_i\vert=\alpha\mathcal{R}_i, b_{t,i}=1\},\\
            \bar{E}_i=&E_{T,i}\setminus (E_i\cup \mathcal{J}_i).
        \end{align*}
        We separately analyze the regret in $E_i$, $\mathcal{J}_i$ and $\bar{E}_i$.\\
        \textbf{Case 1: regret in} $E_i$\\
        For any $f\in\mathbb{H}_i$, the convexity of loss function gives
        \begin{align*}
           &\ell(f_{t,i}({\bm x}_t),y_t)-\ell(f({\bm x}_t),y_t)\\
            \leq& \langle f_{t,i}-f, \nabla_{t,i}\rangle\\
            =&
            \underbrace{\langle f_{t,i}-f'_{t,i},\hat{\nabla}_{t,i}\rangle}_{\Xi_1}
            +\underbrace{\langle f'_{t,i}-f,\nabla_{i(s_t),i}\rangle}_{\Xi_2}+
            \langle f'_{t,i}-f,\nabla_{t,i}-\nabla_{i(s_t),i}\rangle+
            \langle f_{t,i}-f'_{t,i},\nabla_{t,i}-\hat{\nabla}_{t,i}\rangle\\
            =&\Xi_1+\Xi_2+\langle f_{t,i}-f'_{t,i},\nabla_{i(s_t),i}-\hat{\nabla}_{t,i}\rangle+
            \langle f_{t,i}-f,\nabla_{t,i}-\nabla_{i(s_t),i}\rangle\\
            \leq&\left[\mathcal{B}_{\psi_i}(f,f'_{t-1,i})-\mathcal{B}_{\psi_i}(f,f'_{t,i})\right]
            +\underbrace{\left\Vert f_{t,i}-f\right\Vert\cdot\gamma_{t,i}}_{\Xi_3}+
            \underbrace{\left\langle f_{t,i}-f'_{t,i},\nabla_{i(s_t),i}-\hat{\nabla}_{t,i}\right\rangle
            -\mathcal{B}_{\psi_i}(f'_{t,i},f_{t,i})}_{\Xi_4},
        \end{align*}
        where the standard analysis of OMD \cite{Chiang2012Online} gives
        \begin{align*}
            \Xi_1\leq& \mathcal{B}_{\psi_i}(f'_{t,i},f'_{t-1,i})
            -\mathcal{B}_{\psi_i}(f'_{t,i},f_{t,i})-\mathcal{B}_{\psi_i}(f_{t,i},f'_{t-1,i}),\\
            \Xi_2\leq& \mathcal{B}_{\psi_i}(f,f'_{t-1,i})
            -\mathcal{B}_{\psi_i}(f,f'_{t,i})-\mathcal{B}_{\psi_i}(f'_{t,i},f'_{t-1,i}).
        \end{align*}
        Substituting into $\gamma_{t,i}$ and summing over $t\in E_i$ gives
        \begin{align*}
            \sum_{t\in E_i}\Xi_3
            \leq&\sum_{t\in E_i}
            \frac{ \max_t\Vert f_{t,i}-f\Vert_{\mathcal{H}_i}\cdot
            \left\Vert \nabla_{t,i}-\hat{\nabla}_{t,i}\right\Vert^2_{\mathcal{H}_i}}
            {\sqrt{1+\sum_{\tau\leq t} \left\Vert \nabla_{\tau,i}-
            \hat{\nabla}_{\tau,i}\right\Vert^2_{\mathcal{H}_i}\cdot\mathbb{I}_{\nabla_{\tau,i}\neq 0}}}\\
            \leq&2(U+\lambda_i)\cdot\sum_{t\in E_i}
            \frac{
            \left\Vert \nabla_{t,i}-\hat{\nabla}_{t,i}\right\Vert^2_{\mathcal{H}_i}}
            {\sqrt{1+\sum_{\tau\leq t} \left\Vert \nabla_{\tau,i}-
            \hat{\nabla}_{\tau,i}\right\Vert^2_{\mathcal{H}_i}\cdot\mathbb{I}_{\nabla_{\tau,i}\neq 0}}}\\
            \leq& 4(U+\lambda_i)\sqrt{\tilde{\mathcal{A}}_{T,\kappa_i}},
        \end{align*}
        where $\Vert f_{t,i}\Vert_{\mathcal{H}_i}\leq U+\lambda_i$.\\
        According to Lemma \ref{lemma:JCST2022:property_of_OMD},
        we can obtain
        $$
            \sum_{t\in E_i}\Xi_4\leq\frac{\lambda_i}{2}\sum_{t\in E_i}
            \left\Vert \nabla_{i(s_t),i}-\hat{\nabla}_{t,i}\right\Vert^2_{\mathcal{H}_i}
            \leq2\lambda_i\tilde{\mathcal{A}}_{T,\kappa_i}.
        $$
        \textbf{Case 2: regret in} $\bar{E}_i$\\
        We decompose the instantaneous regret as follows,
        \begin{align*}
            &\langle f_{t,i}-f, \nabla_{t,i}\rangle\\
            =&
            \underbrace{\langle f_{t,i}-f'_{t,i},\hat{\nabla}_{t,i}\rangle}_{\Xi_1}
            +\underbrace{\langle f'_{t,i}-f,\tilde{\nabla}_{t,i}\rangle}_{\Xi_2}+
            \underbrace{\langle f_{t,i}-f'_{t,i},\tilde{\nabla}_{t,i}-\hat{\nabla}_{t,i}\rangle}_{\Xi_3}
            +\left\langle f_{t,i}-f, \nabla_{t,i}-\tilde{\nabla}_{t,i}\right\rangle\nonumber\\
            \leq&\mathcal{B}_{\psi_i}(f,f'_{t-1,i})-\mathcal{B}_{\psi_i}(f,f'_{t,i})+
            \langle f_{t,i}-f, \nabla_{t,i}-\tilde{\nabla}_{t,i}\rangle+
            \Xi_3-\left[\mathcal{B}_{\psi_i}(f'_{t,i},f_{t,i})+\mathcal{B}_{\psi_i}(f_{t,i},f'_{t-1,i})\right]\\
            =&\mathcal{B}_{\psi_i}(f,f'_{t-1,i})-\mathcal{B}_{\psi_i}(f,f'_{t,i})+
            \langle f_{t,i}-f, \nabla_{t,i}-\tilde{\nabla}_{t,i}\rangle+
            \Xi_3-\mathcal{B}_{\psi_i}(f'_{t,i},f_{t,i})-
            \frac{\lambda_i}{2}\Vert \hat{\nabla}_{t,i}\Vert^2_{\mathcal{H}_i}\\
            \leq&\mathcal{B}_{\psi_i}(f,f'_{t-1,i})-\mathcal{B}_{\psi_i}(f,f'_{t,i})+
            \langle f_{t,i}-f, \nabla_{t,i}-\tilde{\nabla}_{t,i}\rangle
            -\frac{\lambda_i}{2}\Vert \tilde{\nabla}_{t,i}-\hat{\nabla}_{t,i}\Vert^2_{\mathcal{H}_i}-
            \frac{\lambda_i}{2}\Vert \hat{\nabla}_{t,i}\Vert^2_{\mathcal{H}_i}\quad
            \mathrm{by}~\mathrm{Lemma}~\ref{lemma:JCST2022:property_of_OMD}\\
            =&\mathcal{B}_{\psi_i}(f,f'_{t-1,i})-\mathcal{B}_{\psi_i}(f,f'_{t,i})+
            \langle f_{t,i}-f, \nabla_{t,i}-\tilde{\nabla}_{t,i}\rangle+
            \frac{\lambda_i}{2}\left(\frac{\Vert \nabla_{t,i}-\hat{\nabla}_{t,i}\Vert^2_{\mathcal{H}_i}}
            {(\mathbb{P}[b_{t,i}=1])^2}\mathbb{I}_{b_{t,i}=1}
            -\Vert \hat{\nabla}_{t,i}\Vert^2_{\mathcal{H}_i}\right),
        \end{align*}
        where $\Xi_1+\Xi_2$ follows the analysis in \textbf{Case 1}.\\
        \textbf{Case 3: regret in} $\mathcal{J}_i$\\
        Recalling that the second mirror updating is
        $$
            f'_{t,i}=\mathop{\arg\min}_{f\in\mathbb{H}_i}
            \left\{\langle f,\tilde{\nabla}_{t,i}\rangle+\mathcal{B}_{\psi_{i}}(f,\bar{f}'_{t-1,i}(1))\right\}.
        $$
        We still decompose the instantaneous regret as follows
        $$
            \langle f_{t,i}-f, \nabla_{t,i}\rangle
            =
            \underbrace{\langle f_{t,i}-f'_{t,i},\hat{\nabla}_{t,i}\rangle}_{\Xi_1}
            +\underbrace{\langle f'_{t,i}-f,\tilde{\nabla}_{t,i}\rangle}_{\Xi_2}+
            \underbrace{\langle f_{t,i}-f'_{t,i},\tilde{\nabla}_{t,i}-\hat{\nabla}_{t,i}\rangle}_{\Xi_3}
            +\left\langle f_{t,i}-f, \nabla_{t,i}-\tilde{\nabla}_{t,i}\right\rangle.
        $$
        We reanalyze $\Xi_1$ and $\Xi_2$ as follows
        \begin{align*}
            \Xi_1&\leq \mathcal{B}_{\psi_i}(f'_{t,i},f'_{t-1,i})
            -\mathcal{B}_{\psi_i}(f'_{t,i},f_{t,i})-\mathcal{B}_{\psi_i}(f_{t,i},f'_{t-1,i}),\\
            \Xi_2&\leq
            \mathcal{B}_{\psi_i}(f,\bar{f}'_{t-1,i}(1))
            -\mathcal{B}_{\psi_i}(f,f'_{t,i})-\mathcal{B}_{\psi_i}(f'_{t,i},\bar{f}'_{t-1,i}(1)).
        \end{align*}
        Then $\Xi_1+\Xi_2+\Xi_3$ can be further bounded as follows,
        \begin{align*}
            \Xi_1+\Xi_2+\Xi_3
            \leq& \mathcal{B}_{\psi_i}(f,f'_{t-1,i})-\mathcal{B}_{\psi_i}(f,f'_{t,i})+
            \left[\mathcal{B}_{\psi_i}(f,\bar{f}'_{t-1,i}(1))-\mathcal{B}_{\psi_i}(f,f'_{t-1,i})\right]+\\
            &\left[\mathcal{B}_{\psi_i}(f'_{t,i},f'_{t-1,i})
            -\mathcal{B}_{\psi_i}(f'_{t,i},\bar{f}'_{t-1,i}(1))\right]-
            \left[\mathcal{B}_{\psi_i}(f'_{t,i},f_{t,i})+\mathcal{B}_{\psi_i}(f_{t,i},f'_{t-1,i})\right]
            +\Xi_3.
        \end{align*}
        By Lemma \ref{lemma:JCST2022:property_of_OMD},
        we analyze the following term
        \begin{align*}
            &\Xi_3-\left[\mathcal{B}_{\psi_i}(f'_{t,i},f_{t,i})+\mathcal{B}_{\psi_i}(f_{t,i},f'_{t-1,i})\right]\\
            \leq&\frac{\lambda_i}{2}\left[\frac{\Vert \nabla_{t,i}-\hat{\nabla}_{t,i}\Vert^2_{\mathcal{H}_i}}
            {(\mathbb{P}[b_{t,i}=1])^2}\mathbb{I}_{b_{t,i}=1}
            -\Vert \hat{\nabla}_{t,i}\Vert^2_{\mathcal{H}_i}\right]-
            \frac{1}{2\lambda_i}\Vert f'_{t-1,i}-f'_{t,i}\Vert^2_{\mathcal{H}_i}
            +\langle f'_{t-1,i}-f'_{t,i},\tilde{\nabla}_{t,i}\rangle.
        \end{align*}
        Substituting into the instantaneous regret gives
        \begin{align*}
            \langle f_{t,i}-f,\nabla_{t,i}\rangle
            \leq&\mathcal{B}_{\psi_i}(f,f'_{t-1,i})-\mathcal{B}_{\psi_i}(f,f'_{t,i})+
            \left\langle f_{t,i}-f, \nabla_{t,i}-\tilde{\nabla}_{t,i}\right\rangle
            +\langle \tilde{\nabla}_{t,i},f'_{t-1,i}-f'_{t,i}\rangle+\\
            &\frac{\Vert \bar{f}'_{t-1,i}(1)-f\Vert^2_{\mathcal{H}_i}
            -\Vert f'_{t-1,i}-f\Vert^2_{\mathcal{H}_i}}{2\lambda_i}+
            \frac{\lambda_i}{2}
            \frac{\Vert \nabla_{t,i}-\hat{\nabla}_{t,i}\Vert^2_{\mathcal{H}_i}}
            {(\mathbb{P}[b_{t,i}=1])^2}\mathbb{I}_{b_{t,i}=1}
            -\frac{\lambda_i}{2}\Vert \hat{\nabla}_{t,i}\Vert^2_{\mathcal{H}_i}.
        \end{align*}
        \textbf{Combining all}\\
        Combining the above three cases, we obtain
        \begin{align*}
            \mathcal{T}_2
            \leq&
            \sum_{t\in E_{T,i}}\left[\mathcal{B}_{\psi_i}(f,f'_{t-1,i})-\mathcal{B}_{\psi_i}(f,f'_{t,i})\right]
            +
            4(U+\lambda_i)\tilde{\mathcal{A}}^{\frac{1}{2}}_{T,\kappa_i}+\sum_{t\in \mathcal{J}_i}
            \left[\langle \tilde{\nabla}_{t,i},f'_{t-1,i}-f'_{t,i}\rangle
            +\frac{2U^2}{\lambda_i}\right]+\\
            &\frac{\lambda_i}{2}
            \sum_{t\in \bar{E}_i\cup\mathcal{J}_i}\left[\frac{\Vert \nabla_{t,i}-\hat{\nabla}_{t,i}\Vert^2_{\mathcal{H}_i}}
            {(\mathbb{P}[b_{t,i}=1])^2}\mathbb{I}_{b_{t,i}=1}
            -\Vert \hat{\nabla}_{t,i}\Vert^2_{\mathcal{H}_i}\right]+
            \sum_{t\in \bar{E}_i\cup\mathcal{J}_i}\left\langle f_{t,i}-f, \nabla_{t,i}-\tilde{\nabla}_{t,i}\right\rangle
            +2\lambda_i\tilde{\mathcal{A}}_{T,\kappa_i}.
        \end{align*}
        Recalling that $\Vert f'_{t,i}\Vert_{\mathcal{H}_i}\leq U$ and $f\leq U$.
        Conditioned on $b_{s_{r},i},\ldots,b_{t-1,i}$,
        taking expectation w.r.t. $b_{t,i}$ gives
        \begin{equation}
        \label{eq:JCST2022:lipschitz_loss:algorithm_dependent_complexity}
            \mathbb{E}\left[\mathcal{T}_2\right]\leq\frac{U^2}{2\lambda_i}
            +\left(2U+\frac{2U^2}{\lambda_i}\right)\cdot J
            +\frac{5\lambda_i}{2}\tilde{\mathcal{A}}_{T,\kappa_i}+
            4(U+\lambda_i)\sqrt{\tilde{\mathcal{A}}_{T,\kappa_i}}.
        \end{equation}
        Let $\lambda_i=\frac{\sqrt{K}U}{2\sqrt{B}}$.
        Assuming that $B \geq K$,
        we have $\lambda_i\leq \frac{U}{2}$.
        Then
        \begin{align*}
            \mathbb{E}\left[\mathcal{T}_2\right]
            =&O\left(\frac{U\sqrt{B}}{\sqrt{K}}
            +\frac{\sqrt{K}U}{\sqrt{B}k_1}\tilde{\mathcal{A}}_{T,\kappa_i}
            +U\sqrt{\tilde{\mathcal{A}}_{T,\kappa_i}}\right)
            \quad\quad\quad\mathrm{by}~\eqref{eq:JCST2022:M-OMD-H:J}\\
            =&O\left(\frac{U\sqrt{B}}{\sqrt{K}}
            +\frac{\sqrt{K}U\mathcal{A}_{T,\kappa_i}\ln{T}}{\sqrt{B}k_1}\right),
            \quad\quad\quad\mathrm{by}~\mathrm{Lemma}~\ref{lem:JCST2022:reservoir_estimator}
        \end{align*}
        where we omit the lower order term.

\subsection{Combining $\mathcal{T}_1$ and $\mathcal{T}_2$}

        Combining $\mathcal{T}_1$ and $\mathcal{T}_2$,
        and taking expectation w.r.t. the randomness of reservoir sampling gives
        \begin{align*}
           &\mathbb{E}\left[\mathrm{Reg}(f)\right]\\
            =&\mathbb{E}\left[\sum^{T}_{t=1}
            \ell(f_t({\bm x}_t),y_t)-\sum^{T}_{t=1}\ell(f_{t,i}({\bm x}_t),y_t)\right]
            +\mathbb{E}\left[\mathcal{T}_2\right]\\
            \leq& \frac{3}{\sqrt{2}}\mathbb{E}
            \left[\sqrt{\max_{t,j}c_{t,j}\cdot\sum^{T}_{t=1}\ell(f_{t,i}({\bm x}_t),y_t)\ln{K}}\right]
            +\frac{9}{2}\max_{t,j} c_{t,j}\cdot\ln{K}
            +\mathbb{E}\left[\mathcal{T}_2\right]
            \quad\quad\quad\mathrm{by}~\eqref{eq:JCST2022:kernel_alignment_bound:first_part_of_regret}\\
            =& \frac{3}{\sqrt{2}}\mathbb{E}
            \left[\sqrt{\max_{t,j}c_{t,j}\cdot
            \left(\sum^{T}_{t=1}\ell(f({\bm x}_t),y_t)+\mathbb{E}\left[\mathcal{T}_2\right]\right)\ln{K}}\right]
            +\frac{9}{2}\max_{t,j} c_{t,j}\cdot\ln{K}
            +\mathbb{E}\left[\mathcal{T}_2\right]\\
            =&
             O\left(\sqrt{\max_{t,j}c_{t,j}\cdot L_T(f)\ln{K}}+\frac{U\sqrt{B}}{\sqrt{K}}
            +\frac{\sqrt{K}U\mathcal{A}_{T,\kappa_i}\ln{T}}{\sqrt{B}k_1}
            +\max_{t,j} c_{t,j}\cdot\ln{K}\right).
        \end{align*}
        For the Hinge loss function, we have $\max_{t,j}c_{t,j}=1+U$.
\end{proof}

\section{Proof of Theorem \ref{thm:JCST2022:algorithm_dependent:M-OMD-H}}

\begin{proof}
    For simplicity, denote by
    $$
        \Lambda_{i}=\sum_{t\in \mathcal{J}_i}
        \left[\left\Vert \bar{f}'_{t-1,i}(1)-f\right\Vert^2_{\mathcal{H}_i}
        -\left\Vert f'_{t-1,i}-f\right\Vert^2_{\mathcal{H}_i}\right].
    $$
    There must be a constant $\xi_i\in(0,4]$ such that
    $\Lambda_i\leq \xi_i U^2 J$.
    We will prove a better regret bound if $\xi_i$ is small enough.
    Recalling that \eqref{eq:JCST2022:M-OMD-H:J} gives an upper bound on $J$.
    If $\xi_i\leq \frac{1}{J}$, then
    we rewrite \eqref{eq:JCST2022:lipschitz_loss:algorithm_dependent_complexity} by
    \begin{align*}
        \mathcal{T}_2
        \leq\frac{U^2}{2\lambda_i}+2UJ+\frac{U^2}{2\lambda_i}
        +\frac{5\lambda_i}{2}\tilde{\mathcal{A}}_{T,\kappa_i}
        +4(U+\lambda_i)\sqrt{\tilde{\mathcal{A}}_{T,\kappa_i}}.
    \end{align*}
    Let $\lambda_i=\frac{\sqrt{2}U}{\sqrt{5\tilde{\mathcal{A}}_{T,\kappa_i}}}$.
    Taking expectation w.r.t. the reservoir sampling
    and using Lemma \ref{lem:JCST2022:reservoir_estimator} gives
    $$
        \mathbb{E}\left[\mathcal{T}_2\right]
        =O\left(\frac{UK}{Bk_1}\mathcal{A}_{T,\kappa_i}\ln{T}
        +U\sqrt{\mathcal{A}_{T,\kappa_i}\ln{T}}\right),
    $$
    where we omit the lower order terms.
    Combining $\mathcal{T}_1$ and $\mathcal{T}_2$ gives
    \begin{align*}
       &\mathbb{E}\left[\mathrm{Reg}(f)\right]\\
        =& \frac{3}{\sqrt{2}}\mathbb{E}
        \left[\sqrt{\max_{t,j}c_{t,j}\cdot
        \left(\sum^{T}_{t=1}\ell(f({\bm x}_t),y_t)+\mathbb{E}\left[\mathcal{T}_2\right]\right)\ln{K}}\right]
        +\frac{9}{2}\max_{t,j} c_{t,j}\cdot\ln{K}
        +\mathbb{E}\left[\mathcal{T}_2\right]\\
        =&
         O\left(\sqrt{\max_{t,j}c_{t,j}\cdot L_T(f)\ln{K}}+\frac{UK}{Bk_1}\mathcal{A}_{T,\kappa_i}\ln{T}
    +U\sqrt{\mathcal{A}_{T,\kappa_i}\ln{T}}
        +\max_{t,j} c_{t,j}\cdot\ln{K}\right),
    \end{align*}
    which concludes the proof.
\end{proof}

\section{Proof of Lemma \ref{lem:JCST2022:number_epoches_smooth_loss_functions}}

    \begin{proof}
        Recalling the definition of $\mathcal{J}$ and $T_r$
        in Section \ref{sec:JCST2022:proof_Lemma_1}.
        For any $ t\in T_r$,
        $({\bm x}_t,y_t)$ will be added into $S$ only if $b_t=1$.
        At the end of the $t_{r}$-th round,
        we have
        $$
            \vert S\vert = \frac{B}{2}\mathbb{I}_{r\neq 1}+\sum^{t_{r}}_{t=t_{r-1}+1}b_{t} = B.
        $$
        We remove $\frac{B}{2}$ examples from $S$ at the end of the $t_r$-th round.

        Assuming that there is no budget.
        For any $t_0>t_{r-1}+1$,
        we will prove an upper bound on $\sum^{t_0}_{t=t_{r-1}+1}b_t$.
        Define a random variable $X_t$ as follows,
        $$
            X_t = b_t-\mathbb{P}[b_t=1],\quad \vert X_t\vert \leq 1.
        $$
        Under the condition of $b_{t_{r-1}+1},\ldots,b_{t-1}$,
        we have $\mathbb{E}_{b_t}[X_t]=0$.
        Thus $X_{t_{r-1}+1},\ldots,X_{t_0}$ form bounded martingale difference.
        Let $\hat{L}_{a:b}:=\sum^b_{t=a}\ell(f_t({\bm x}_t),y_t)$ and $\hat{L}_{1:T}\leq N$.
        The sum of conditional variances satisfies
        $$
            \Sigma^2
            \leq \sum^{t_0}_{t=t_{r-1}+1}\frac{\vert \ell'(f_{t}({\bm x}_t),y_t)\vert}
            {\vert \ell'(f_{t}({\bm x}_t),y_t)\vert+G_1}
            \leq\frac{G_2}{G_1}\hat{L}_{t_{r-1}+1:t_0},
        $$
        where the last inequality comes from
        Assumption \ref{ass:JCST2022:property_smooth_loss}.
        Since $\hat{L}_{t_{r-1}+1:t_0}$ is a random variable,
        Lemma \ref{lem:JCST2022:improved:Bernstein_ineq_for martingales} can give an upper bound on
        $\sum^{t_0}_{t=t_{r-1}+1}b_{t}$
        with probability at least $1-2\lceil\log{N}\rceil\delta$.
        Let $t_{r}$ be the minimal $t_0$ such that
        \begin{align*}
            \frac{G_2}{G_1}\hat{L}_{t_{r-1}+1:t_r}+\frac{2}{3}\ln\frac{1}{\delta}
            +2\sqrt{\frac{G_2}{G_1}\hat{L}_{t_{r-1}+1:t_r}\ln\frac{1}{\delta}}
            \geq \frac{B}{2}\cdot\mathbb{I}_{r\geq 2}+B\cdot\mathbb{I}_{r=1}.
        \end{align*}
        The $r$-th epoch will end at $t_r$.
        Summing over $r\in\{1,\ldots,J\}$,
        with probability at least $1-2J\lceil\log{N}\rceil\delta$,
        \begin{align*}
            \sum^{J}_{r=1}\sum^{t_r}_{t=t_{r-1}+1}b_{t}\leq
            \sum^J_{r=1}
            \left(\frac{G_2}{G_1}\hat{L}_{t_{r-1}+1:t_r}+\frac{2}{3}\ln\frac{1}{\delta}
                +2\sqrt{\frac{G_2}{G_1}\hat{L}_{t_{r-1}+1:t_r}\ln\frac{1}{\delta}}\right),
        \end{align*}
        which is equivalent to
        $$
            \frac{B}{2}+\frac{JB}{2}
            \leq \frac{G_2}{G_1}\hat{L}_{1:T}+\frac{2}{3}J\ln\frac{1}{\delta}
            +2\sqrt{J\frac{G_2}{G_1}\hat{L}_{1:T}\ln\frac{1}{\delta}}.
        $$
        Solving the above inequality yields,
        \begin{align*}
            J\leq& \frac{2G_2}{G_1}\frac{\hat{L}_{1:T}}{B-\frac{4}{3}\ln\frac{1}{\delta}}
            +\frac{16G_2}{G_1}\frac{\hat{L}_{1:T}}{(B-\frac{4}{3}\ln\frac{1}{\delta})^2}\ln\frac{1}{\delta}+
            \frac{4\sqrt{2}}{(B-\frac{4}{3}\ln\frac{1}{\delta})^{\frac{3}{2}}}\frac{G_2}{G_1}
            \hat{L}_{1:T}\sqrt{\ln\frac{1}{\delta}}.
        \end{align*}
        Let $B\geq21\ln\frac{1}{\delta}$.
        Simplifying the above result concludes the proof.
    \end{proof}

\section{Proof of Theorem \ref{thm:JCST2022:small_loss_bound:M-OMD-S}}

    \begin{proof}
        Let ${\bm p}\in\Delta_{K-1}$ satisfy $p_i=1$.
        By the convexity of loss function, we have
        \begin{align*}
           \mathrm{Reg}(f)
            \leq&\sum^{T}_{t=1}\langle\ell'(f_t({\bm x}_t),y_t),f_t({\bm x}_t)-f({\bm x}_t)\rangle\\
            =&\sum^{T}_{t=1}
            \left\langle\ell'(f_t({\bm x}_t),y_t),\sum^K_{i=1}p_{t,i}f_{t,i}({\bm x}_t)-f({\bm x}_t)\right\rangle\\
            =&\sum^{T}_{t=1}
            \left\langle\ell'(f_t({\bm x}_t),y_t),\sum^K_{i=1}p_{t,i}f_{t,i}({\bm x}_t)
            -\sum^K_{i=1}p_if_{t,i}({\bm x}_t)
            +\sum^K_{i=1}p_if_{t,i}({\bm x}_t)-f({\bm x}_t)\right\rangle\\
            =&\underbrace{\sum^{T}_{t=1}\ell'(f_t({\bm x}_t),y_t)
            \sum^K_{i=1}(p_{t,i}-p_i)f_{t,i}({\bm x}_t)}_{\mathcal{T}_{1}}
            +\underbrace{\sum^{T}_{t=1}\langle\nabla_{t,i},f_{t,i}-f\rangle}_{\mathcal{T}_{2}}.
        \end{align*}
        We first analyze $\mathcal{T}_1$.
        We have
        \begin{align*}
            \mathcal{T}_{1}
            =&
            \sum_{t\in T^1}\sum^K_{i=1}(p_{t,i}-p_i)\cdot
            \ell'(f_t({\bm x}_t),y_t)\cdot
            \left(f_{t,i}({\bm x}_t)-\min_{j\in[K]}f_{t,j}({\bm x}_t)\right)+\\
            &\sum_{t\in T^2}\sum^K_{i=1}(p_{t,i}-p_i)\cdot
            \ell'(f_t({\bm x}_t),y_t)\cdot
            \left(f_{t,i}({\bm x}_t)-\max_{j\in[K]}f_{t,j}({\bm x}_t)\right) \\
            =&\sum^{T}_{t=1}\langle {\bm p}_t-{\bm p},{\bm c}_t\rangle
            \quad\quad\quad\mathrm{by}~\eqref{eq:JCST2022:smooth_loss:unsigned_criterion}\\
            \leq&\frac{3}{\sqrt{2}}
                \sqrt{\max_{t,j}c_{t,j}\sum^{T}_{\tau=1}\langle{\bm p}_\tau,{\bm c}_\tau\rangle\ln{K}}
                \quad\quad\quad\mathrm{by}~\eqref{eq:JCST2023_supp:initial_regret_PEA}\\
            \leq&\frac{3}{\sqrt{2}}
                \sqrt{\max_{t,j}c_{t,j}\sum^{T}_{\tau=1}\vert\ell'(f_t({\bm x}_t),y_t)\vert
                \cdot 2U\ln{K}}\\
            \leq&\frac{6}{\sqrt{2}}U\sqrt{G_2G_1\hat{L}_{1:T}\ln{K}}.
            \quad\quad\quad\mathrm{by}~\mathrm{Assumption}~\eqref{ass:JCST2022:property_smooth_loss}
        \end{align*}
        Next we analyze $\mathcal{T}_2$.
        We decompose $[T]$ as follows,
        \begin{align*}
            T_1=&\{t\in[T]:\mathrm{con} (a)\},\\
            \mathcal{J}=&\{t\in [T]:\vert S\vert=\alpha\mathcal{R}, b_t=1\},\\
            \bar{T}_1=&[T]\setminus (T_1\cup \mathcal{J}).
        \end{align*}
        \textbf{Case 1: regret in} $T_1$\\
        We decompose $\langle f_{t,i}-f,\nabla_{t,i}\rangle$ as follows,
        \begin{align*}
            &\langle f_{t,i}-f,\nabla_{t,i}\rangle\\
            =&\langle f_{t+1,i}-f,\nabla_{i(s_t),i}\rangle+
            \langle f_{t+1,i}-f,\nabla_{t,i}-\nabla_{i(s_t),i}\rangle+
            \langle f_{t,i}-f_{t+1,i},\nabla_{i(s_t),i}+\nabla_{t,i}-\nabla_{i(s_t),i}\rangle\\
            \leq&
            \mathcal{B}_{\psi_i}(f,f_{t,i})-\mathcal{B}_{\psi_i}(f,f_{t+1,i})
            -\mathcal{B}_{\psi_i}(f_{t+1,i},f_{t,i})+
            \left\langle f_{t,i}-f_{t+1,i},\nabla_{i(s_t),i}\right\rangle
            +\left\langle f_{t,i}-f,\nabla_{t,i}-\nabla_{i(s_t),i}\right\rangle\\
            \leq&\mathcal{B}_{\psi_i}(f,f_{t,i})-\mathcal{B}_{\psi_i}(f,f_{t+1,i})
            +\frac{\lambda}{2}\Vert \nabla_{i(s_t),i}\Vert^2_{\mathcal{H}_i}+
            \left\langle f_{t,i}-f,\nabla_{t,i}-\nabla_{i(s_t),i}\right\rangle,
        \end{align*}
        where the last inequality comes from Lemma \ref{lemma:JCST2022:property_of_OMD}.
        Next we analyze the third term.
        \begin{align*}
            \sum_{t\in T_1}\langle f_{t,i}-f,\nabla_{t,i}-\nabla_{i(s_t),i}\rangle
            \leq& 2U\cdot\sum_{t\in T_1}
            \frac{\left\vert \ell'(f_t({\bm x}_t),y_t)\right\vert}
            {\sqrt{1+\sum_{\tau\in T_1,\tau\leq t}
            \left\vert\ell'(f_{\tau}({\bm x}_{\tau}),y_{\tau})\right\vert}}\\
            \leq&4U\sqrt{G_2\hat{L}_{1:T}}.
        \end{align*}
        \textbf{Case 2: regret in} $\bar{T}_1$\\
        We use a different decomposition as follows
        \begin{align*}
            &\langle f_{t,i}-f,\nabla_{t,i}\rangle\\
            =&\underbrace{\langle f_{t+1,i}-f,\tilde{\nabla}_{t,i}\rangle}_{\Xi_1}+
            \underbrace{\langle f_{t+1,i}-f,\nabla_{t,i}-\tilde{\nabla}_{t,i}\rangle}_{\Xi_2}
            +\underbrace{\langle f_{t,i}-f_{t+1,i},\nabla_{t,i}\rangle}_{\Xi_3}\\
            \leq&\mathcal{B}_{\psi_i}(f,f_{t,i})
            -\mathcal{B}_{\psi_i}(f,f_{t+1,i})-\mathcal{B}_{\psi_i}(f_{t+1,i},f_{t,i})
            +\langle f_{t+1,i}-f,\nabla_{t,i}-\tilde{\nabla}_{t,i}\rangle+\\
            &\langle f_{t,i}-f_{t+1,i},\tilde{\nabla}_{t,i}\rangle
            +\langle f_{t,i}-f_{t+1,i},\nabla_{t,i}-\tilde{\nabla}_{t,i}\rangle\\
            =&\mathcal{B}_{\psi_i}(f,f_{t,i})-\mathcal{B}_{\psi_i}(f,f_{t+1,i})+
            \left\langle f_{t,i}-f,\nabla_{t,i}-\tilde{\nabla}_{t,i}\right\rangle
            +
            \langle f_{t,i}-f_{t+1,i},\tilde{\nabla}_{t,i}\rangle
            -\mathcal{B}_{\psi_i}(f_{t+1,i},f_{t,i})\\
            \leq&\mathcal{B}_{\psi_i}(f,f_{t,i})-\mathcal{B}_{\psi_i}(f,f_{t+1,i})+
            \left\langle f_{t,i}-f,\nabla_{t,i}-\tilde{\nabla}_{t,i}\right\rangle
            +\frac{\lambda_i}{2}\Vert \tilde{\nabla}_{t,i}\Vert^2_{\mathcal{H}_i}.
        \end{align*}
        \textbf{Case 3: regret in} $\mathcal{J}$\\
        We decompose $\left\langle f_{t,i}-f,\nabla_{t,i}\right\rangle$
        into three terms as
        in \textbf{Case 2}.
        The second mirror updating is
        $$
            f_{t+1,i}=\mathop{\arg\min}_{f\in\mathbb{H}_i}
            \left\{\langle f,\tilde{\nabla}_{t,i}\rangle+\mathcal{B}_{\psi_{i}}(f,\bar{f}_{t,i}(2))\right\}.
        $$
        Similar to the analysis of \textbf{Case 2}, we obtain
        \begin{align*}
            \Xi_1\leq& \mathcal{B}_{\psi_i}(f,f_{t,i})-\mathcal{B}_{\psi_i}(f,f_{t+1,i})
            -\mathcal{B}_{\psi_i}(f_{t+1,i},\bar{f}_{t,i}(2))+
            [\mathcal{B}_{\psi_i}(f,\bar{f}_{t,i}(2))-\mathcal{B}_{\psi_i}(f,f_{t,i})],\\
            \Xi_3=& \langle \bar{f}_{t,i}(2)-f_{t+1,i},\tilde{\nabla}_{t,i}\rangle
            +\langle f_{t,i}-\bar{f}_{t,i}(2),\tilde{\nabla}_{t,i}\rangle+
            \langle f_{t,i}-f_{t+1,i},\nabla_{t,i}-\tilde{\nabla}_{t,i}\rangle.
        \end{align*}
        Combining $\Xi_1$, $\Xi_2$ and $\Xi_3$ gives
        \begin{align}
            &\left\langle f_{t,i}-f,\nabla_{t,i}\right\rangle\nonumber\\
            \leq& \mathcal{B}_{\psi_i}(f,f_{t,i})-\mathcal{B}_{\psi_i}(f,f_{t+1,i})
            +\langle f_{t,i}-f,\nabla_{t,i}-\tilde{\nabla}_{t,i}\rangle+\nonumber\\
            &\underbrace{\mathcal{B}_{\psi_i}(f,\bar{f}_{t,i}(2))-\mathcal{B}_{\psi_i}(f,f_{t,i})
            +\langle f_{t,i}-\bar{f}_{t,i}(2),\tilde{\nabla}_{t,i}\rangle}_{\Xi_4}+
            \langle \bar{f}_{t,i}(2)-f_{t+1,i},\tilde{\nabla}_{t,i}\rangle
            -\mathcal{B}_{\psi_i}(f_{t+1,i},\bar{f}_{t,i}(2))
            \label{eq:appendix_Xi_4}\\
            \leq&\mathcal{B}_{\psi_i}(f,f_{t,i})-\mathcal{B}_{\psi_i}(f,f_{t+1,i})
            +\langle f_{t,i}-f,\nabla_{t,i}-\tilde{\nabla}_{t,i}\rangle+\nonumber\\
            &\frac{2U^2}{\lambda_i}+4UG_1+\langle \bar{f}_{t,i}(2)-f_{t+1,i},\tilde{\nabla}_{t,i}\rangle
            -\mathcal{B}_{\psi_i}(f_{t+1,i},\bar{f}_{t,i}(2))\quad\quad\quad
            \mathrm{by}~\mathrm{Lemma}~\ref{lemma:JCST2022:property_of_OMD}\nonumber\\
            \leq&\mathcal{B}_{\psi_i}(f,f_{t,i})-\mathcal{B}_{\psi_i}(f,f_{t+1,i})
            +\langle f_{t,i}-f,\nabla_{t,i}-\tilde{\nabla}_{t,i}\rangle+
            \frac{2U^2}{\lambda_i}+4UG_1+\frac{\lambda_i}{2}\Vert\tilde{\nabla}_{t,i}\Vert^2_{\mathcal{H}_i}.\nonumber
        \end{align}
        Combining the regret in $T_1$, $\mathcal{J}$ and $\bar{T}_1$ gives
        \begin{align*}
            \mathcal{T}_2\leq&4U\sqrt{G_2\hat{L}_{1:T}}
            +\left(\frac{2U^2}{\lambda_i}+4UG_1\right) \vert \mathcal{J}\vert+
            \underbrace{\sum_{t\in \bar{T}_1\cup \mathcal{J}}
            \langle f_{t,i}-f,\nabla_{t,i}-\tilde{\nabla}_{t,i}\rangle}_{\Xi_{2,1}}+\\
            &\sum^T_{t=1}\left(\mathcal{B}_{\psi_i}(f,f_{t,i})-\mathcal{B}_{\psi_i}(f,f_{t+1,i})\right)
            +\lambda_i\underbrace{\left(\frac{1}{2}\sum_{t\in \bar{T}_1\cup \mathcal{J}}
            \Vert \tilde{\nabla}_{t,i}\Vert^2_{\mathcal{H}_i}
            +\sum_{t\in T_1}\frac{1}{2}\Vert \nabla_{i(s_t),i}\Vert^2_{\mathcal{H}_i}\right)}_{\Xi_{2,2}}\\
            \leq&
            4U\sqrt{G_2\hat{L}_{1:T}}
            +\left(\frac{U^2}{2\lambda_i}+4UG_1\right) \vert \mathcal{J}\vert+
            \Xi_{2,1}+\frac{U^2}{2\lambda_i}+\Xi_{2,2}.
        \end{align*}
        Lemma \ref{lem:JCST2022:the_Second_Moment_of_Gradient_smooth_loss_v1} gives,
        with probability at least $1-\Theta(\lceil\ln{T}\rceil)\delta$,
        \begin{align*}
        \Xi_{2,1}\leq& \frac{4}{3}UG_1\ln\frac{1}{\delta}
            +2U\sqrt{2G_2G_1\hat{L}_{1:T}\ln\frac{1}{\delta}},\\
        \Xi_{2,2}
            \leq&G_1G_2\hat{L}_{1:T}+
            \frac{2}{3}G^2_1\ln\frac{1}{\delta}
            +2\sqrt{G^3_1G_2\hat{L}_{1:T}\ln\frac{1}{\delta}}.
        \end{align*}

        Let $\lambda_i=\frac{2U}{\sqrt{B}G_1}$.
        Using Lemma \ref{lem:JCST2022:number_epoches_smooth_loss_functions}
        and combining $\mathcal{T}_1$ and $\mathcal{T}_2$ gives,
        with probability at least $1-\Theta(\lceil\ln{T}\rceil)\delta$,
        \begin{align*}
            \mathrm{Reg}(f)=&\hat{L}_{1:T}-L_{T}(f)\\
            \leq&\mathcal{T}_1 + \mathcal{T}_2\\
            \leq& 10U\sqrt{G_2G_1\hat{L}_{1:T}\ln\frac{1}{\delta}}+
            \frac{UG_1\sqrt{B}}{4}+\frac{6UG_2\hat{L}_{1:T}}{\sqrt{B-\frac{4}{3}\ln\frac{1}{\delta}}},
        \end{align*}
        where we omit the constant terms and the lower order terms.
        Let $\gamma=\frac{6UG_2}{\sqrt{B-\frac{4}{3}\ln\frac{1}{\delta}}}$
        and $U\leq\frac{1}{8G_2}\sqrt{B-\frac{4}{3}\ln\frac{1}{\delta}}$.
        Then $1-\gamma\geq \frac{1}{4}$.
        Solving for $\hat{L}_{1:T}$ concludes the proof.

        Finally,
        we explain why it must be satisfied that $K\leq d$.
        The space complexity of M-OMD-S is $O(KB+dB+K)$.
        According to Assumption \ref{ass:JCST2022:reduction},
        the coefficient $\alpha$ only depends on $d$.
        If $K\leq d$, then the space complexity of M-OMD-S is $O(dB)$.
        In this case, $B=\Theta(\alpha\mathcal{R})$.
        If $K>d$, then the space complexity is $O(KB)$.
        M-OMD-S must allocate the memory resource over $K$ hypotheses.
        For instance, if $K=d^{\nu}$, $\nu>1$, then $B=\Theta(K^{\frac{1-v}{v}}\alpha\mathcal{R})$.
        Thus the regret bound will increase a factor of order $O(K^{\frac{v-1}{2v}})$.
    \end{proof}

\section{Proof of Theorem \ref{thm:JCST2022:lower_bound_small_loss}}

\begin{proof}
    Let $\kappa({\bm x},{\bm v})=\langle {\bm x},{\bm v}\rangle^p$.
    The adversary first constructs $\mathcal{I}_T$.
    For $1\leq t\leq 3B$,
    let ${\bm x}_t={\bm e}_t$
    where ${\bm e}_t$ is the standard basis vector in $\mathbb{R}^d$.
    Let $y_t=1$ if $t$ is odd. Otherwise, $y_t=-1$.
    For $3B+1\leq t\leq T$,
    let $({\bm x}_{t},y_t)\in \{({\bm x}_{\tau},y_{\tau})\}^{3B}_{\tau=1}$ uniformly.

    We construct a competitor as follows,
    $$
        \bar{f}_{\mathbb{H}}
        =\frac{U}{\sqrt{3B}}\cdot\sum^{3B}_{t=1}y_t\kappa({\bm x}_t,\cdot).
    $$
    It is easy to prove
    \begin{align*}
        L_T(\bar{f}_{\mathbb{H}})=&T\cdot\ln\left(1+\exp\left(-\frac{U}{\sqrt{3B}}\right)\right),\\
        \Vert \bar{f}_{\mathbb{H}}\Vert_{\mathcal{H}}=&U.
    \end{align*}
    Thus $\bar{f}_{\mathbb{H}}\in\mathbb{H}$.\\
    Let $\mathcal{A}$ be an algorithm storing $B$ examples at most.
    At the beginning of round $t$,
    let
    $$
        f_t=\sum_{i\leq B} a^{(t)}_{i}\kappa({\bm x}^{(t)}_i,\cdot)
    $$
    be the hypothesis maintained by $\mathcal{A}$
    where ${\bm x}^{(t)}_i\in\{{\bm x}_1,\ldots,{\bm x}_{t-1}\}$.
    Besides,
    it must be satisfied that
    \begin{equation}
    \label{eq:JCST2022:norm_constraint}
        \Vert f_t\Vert_{\mathcal{H}}=\sqrt{\sum_{i\leq B}\vert a^{(t)}_i\vert^2}\leq U.
    \end{equation}
    It is easy to obtain
    $
        \sum^{3B}_{t=1}\ell(f_t({\bm x}_t),y_t)= 3\ln(2)B.
    $
    For any $t\geq 3B+1$,
    the expected per-round loss is
    $$
        \mathbb{E}\left[\ell(f_t({\bm x}_t),y_t)\right]
        =\frac{2}{3}\ln(2)+\frac{1}{3B}\sum_{i\leq B}\ln(1+\exp(-\vert a^{(t)}_i\vert)).
    $$
    Note that $\vert a^{(t)}_1\vert,\ldots,\vert a^{(t)}_B\vert$ must satisfy \eqref{eq:JCST2022:norm_constraint}.
    By the Lagrangian multiplier method,
    the minimum is obtained at $\vert a^{(t)}_i\vert=\frac{U}{\sqrt{B}}$ for all $i$.
    Then we have
    $$
        \mathbb{E}\left[\ell(f_t({\bm x}_t),y_t)\right]
        \geq\frac{2}{3}\ln(2)+\frac{\ln(1+\exp(-\frac{U}{\sqrt{B}}))}{3}.
    $$
    It can be verified that
    $$
        \forall 0<x\leq 0.2,\quad\ln(1+\exp(-x))\leq \ln(2)-0.45x.
    $$
    Let $B< T$ and $U\leq \frac{1}{5}\sqrt{3B}$.
    The expected regret w.r.t. $\bar{f}_{\mathbb{H}}$ is lower bounded as follows
    \begin{align*}
        \mathbb{E}\left[\mathrm{Reg}(\bar{f}_{\mathbb{H}})\right]
        \geq& 3\ln(2)B+(T-3B)\cdot\frac{\ln(1+\exp(-\frac{U}{\sqrt{B}}))}{3}+
        \frac{2}{3}\ln(2)\cdot(T-3B)-
        T\cdot\ln\left(1+\exp\left(-\frac{U}{\sqrt{3B}}\right)\right)\\
        =&\left(\frac{2}{3}T+B\right)\cdot\left(\ln(2)-\ln\left(1+\exp\left(-\frac{U}{\sqrt{3B}}\right)\right)\right)+
        \frac{1}{3}(T-3B)\ln\frac{1+\exp(-\frac{U}{\sqrt{B}})}{1+\exp(-\frac{U}{\sqrt{3B}})}\\
        \geq&\frac{\sqrt{3}}{10}\cdot \frac{UT}{\sqrt{B}}+\frac{1}{3}\left(\frac{\sqrt{3}}{3}-1\right)\cdot \frac{UT}{\sqrt{B}}.
    \end{align*}
    It can be verified that $L_T(\bar{f}_{\mathbb{H}})=\Theta(T)$.
    Replacing $T$ with $\Theta(L_T(\bar{f}_{\mathbb{H}}))$ concludes the proof.
\end{proof}

\section{Proof of Theorem \ref{thm:JCST2022:algorithm_dependent:M-OMD-S}}

\begin{proof}
    There is a $\xi_i\in(0,4]$ such that
    $$
        \Lambda_i=\sum_{t\in \mathcal{J}}
        \left[\left\Vert \bar{f}_{t,i}(2)-f\right\Vert^2_{\mathcal{H}_i}
        -\Vert f_{t,i}-f\Vert^2_{\mathcal{H}_i}\right]
        \leq \xi_i U^2\vert \mathcal{J}\vert.
    $$
    We can rewrite $\Xi_4$ in \eqref{eq:appendix_Xi_4} as follow,
    $$
        \sum_{t\in\mathcal{J}}
        \Xi_4\leq \left(\frac{\xi_i U^2}{2\lambda_i}+4UG_1\right)\cdot \vert \mathcal{J}\vert.
    $$
    If $\xi_i\leq\frac{1}{\vert\mathcal{J}\vert}$, then
    $\frac{\xi_i U^2}{2\lambda_i}\cdot\vert\mathcal{J}\vert\leq\frac{U^2}{2\lambda_i}$.
    Let $\lambda_i=\frac{U}{\sqrt{G_1G_2\hat{L}_{1:T}}}$.
    In this way,
    we obtain a new upper bound on $\mathcal{T}_2$.
    Combining $\mathcal{T}_1$ and $\mathcal{T}_2$ gives
    \begin{align*}
        \hat{L}_{1:T}-L_{T}(f)
        \leq 12U\sqrt{G_2G_1\hat{L}_{1:T}\ln\frac{1}{\delta}}+
        \frac{16UG_2\hat{L}_{1:T}}{B-\frac{4}{3}\ln\frac{1}{\delta}}+4UG_1\ln\frac{1}{\delta}.
    \end{align*}
    Let $\gamma=16UG_2(B-\frac{4}{3}\ln\frac{1}{\delta})^{-1}$
    and $U<\frac{B-\frac{4}{3}\ln\frac{1}{\delta}}{32G_2}$.
    We have $\gamma\leq\frac{1}{2}$.
    Solving for $\hat{L}_{1:T}$ concludes the proof.
\end{proof}

\section{Auxiliary Lemmas}

    \begin{lemmaAppend}[\cite{Hazan2009Better}]
    \label{lem:JCST2022:variance_reservoir_estimator}
        $\forall t> M$ and $\forall i\in[K]$,
        $\mathbb{E}[\Vert\hat{\nabla}_{t,i}-\mu_{t,i}\Vert^2_{\mathcal{H}_i}]
        \leq\frac{1}{t\vert V\vert}\mathcal{A}_{t,\kappa_i}$.
    \end{lemmaAppend}

    \begin{lemmaAppend}
    \label{lem:JCST2022:analysis_second_moment_TV}
        Let $\eta_t$ follow \eqref{eq:JCST2022:updating_sampling_probability}
        and ${\bm p}_1$ be the uniform distribution.
        Then
        $$
            \sum^T_{t=1}\sum^K_{i=1}\frac{\eta_t}{2}p_{t,i}c^2_{t,i}
            \leq \sqrt{2\ln{K}}
            \sqrt{\sum^{T}_{\tau=1}\sum^K_{i=1}p_{\tau,i}c^2_{\tau,i}}
            +\frac{\sqrt{2\ln{K}}}{2}\max_{t,i}c_{t,i}.
        $$
    \end{lemmaAppend}
    \begin{proof}
    Let $\sigma_\tau=\sum^K_{i=1}p_{\tau,i}c^2_{\tau,i}$ and $\sigma_0=1$.
    We decompose the term as follows
    \begin{align*}
        \frac{\sqrt{\ln{K}}}{\sqrt{2}}\cdot\sum^T_{t=1}\frac{\sigma_t}{\sqrt{1+\sum^{t-1}_{\tau=1}\sigma_\tau}}
        =&\frac{\sqrt{\ln{K}}}{\sqrt{2}}\cdot\sum^T_{t=1}\frac{\sigma_t}{\sqrt{\sum^{t-1}_{\tau=0}\sigma_\tau}}\\
        =&\frac{\sqrt{\ln{K}}}{\sqrt{2}}\cdot\left[\sigma_1
        +\sum^T_{t=2}\frac{\sigma_{t-1}}{\sqrt{1+\sum^{t-1}_{\tau=1}\sigma_\tau}}
        +\sum^T_{t=2}\frac{\sigma_t-\sigma_{t-1}}{\sqrt{1+\sum^{t-1}_{\tau=1}\sigma_\tau}}\right].
    \end{align*}
    We analyze the third term.
    \begin{align*}
        \sum^T_{t=2}\frac{\sigma_{t}-\sigma_{t-1}}{\sqrt{1+\sum^{t-1}_{\tau=1}\sigma_\tau}}
        =&\frac{-\sigma_{1}}{\sqrt{1+\sigma_1}}
        +\frac{\sigma_{T}}{\sqrt{1+\sum^{T-1}_{\tau=1}\sigma_\tau}}+
        \sum^{T-1}_{t=2}\sigma_t\left[\frac{1}{\sqrt{1+\sum^{t-1}_{\tau=1}\sigma_\tau}}
        -\frac{1}{\sqrt{1+\sum^{t}_{\tau=1}\sigma_\tau}}\right]\\
        \leq&-\frac{\sigma_{1}}{\sqrt{1+\sigma_1}}
        +\max_{t=1,\ldots,T}\sigma_t\cdot\frac{1}{\sqrt{1+\sigma_1}}.
    \end{align*}
    Now we analyze the second term.
    \begin{align*}
        \sum^T_{t=2}\frac{\sigma_{t-1}}{\sqrt{1+\sum^{t-1}_{\tau=1}\sigma_\tau}}
        =\frac{\sigma_{1}}{\sqrt{1+\sigma_1}}
        +\sum^T_{t=3}\frac{\sigma_{t-1}}{\sqrt{1+\sum^{t-1}_{\tau=1}\sigma_\tau}}.
    \end{align*}
        For any $a> 0$ and $b> 0$,
        we have $2\sqrt{a}\sqrt{b}\leq a+b$.
        Let $a=1+\sum^{t-1}_{\tau=1}\sigma_\tau$ and
            $b=1+\sum^{t-2}_{\tau=1}\sigma_\tau$.
        Then we have
        \begin{align*}
            2\sqrt{1+\sum^{t-1}_{\tau=1}\sigma_\tau}
            \cdot\sqrt{1+\sum^{t-2}_{\tau=1}\sigma_\tau}
            \leq 2\left(1+\sum^{t-1}_{\tau=1}\sigma_\tau\right)-\sigma_{t-1}.
        \end{align*}
        Dividing by $\sqrt{a}$ and rearranging terms yields
        \begin{align*}
            \frac{1}{2}\frac{\sigma_{t-1}}{\sqrt{1+\sum^{t-1}_{\tau=1}\sigma_\tau}}
            \leq \sqrt{1+\sum^{t-1}_{\tau=1}\sigma_\tau}
            -\sqrt{1+\sum^{t-2}_{\tau=1}\sigma_\tau}.
        \end{align*}
        Summing over $t=3,\ldots,T$,
        we obtain
        \begin{align*}
            &\sum^T_{t=3}\frac{\sigma_{t-1}}{\sqrt{1+\sum^{t-1}_{\tau=1}\sigma_\tau}}
            \leq 2\sqrt{1+\sum^{T-1}_{\tau=1}\sigma_\tau}
            -2\sqrt{1+\sigma_1}.
        \end{align*}
        Summing over all results, we have
        \begin{align*}
            \sum^T_{t=1}\frac{\sigma_t}{\sqrt{1+\sum^{t-1}_{\tau=1}\sigma_\tau}}
            \leq &2\sqrt{1+\sum^{T-1}_{\tau=1}\sigma_\tau}
            -2\sqrt{1+\sigma_1}+\max_t\sigma_t\cdot\frac{1}{\sqrt{1+\sigma_1}}+\sigma_1\\
%            \leq&2\sqrt{\sum^{T}_{\tau=1}\sigma_\tau}+\max_t\sigma_t\cdot\frac{1}{\sqrt{1+\sigma_1}}+\sigma_1\\
            \leq&2\sqrt{\sum^{T}_{\tau=1}\sigma_\tau}+\max_t\sigma_t,
        \end{align*}
        which concludes the proof.
    \end{proof}

    \begin{lemmaAppend}[Improved Bernstein's inequality \cite{Li2024On}]
    \label{lem:JCST2022:improved:Bernstein_ineq_for martingales}
        Let $X_1,\ldots,X_n$ be a bounded martingale difference sequence w.r.t. the filtration
        $\mathcal{F}=(\mathcal{F}_k)_{1\leq k\leq n}$ and with $\vert X_k\vert\leq a$.
        Let $Z_t=\sum^t_{k=1}X_{k}$ be the associated martingale.
        Denote the sum of the conditional variances by
        $$
            \Sigma^2_n=\sum^n_{k=1}\mathbb{E}\left[X^2_k\vert\mathcal{F}_{k-1}\right]\leq v,
        $$
        where $v\in[0,V]$ is a random variable and $V\geq 2$ is a constant.
        Then for all constants $a>0$,
        with probability at least $1-2\lceil\log{V}\rceil\delta$,
        $$
            \max_{t=1,\ldots,n}Z_t <
            \frac{2a}{3}\ln\frac{1}{\delta}+\sqrt{\frac{2}{V}\ln\frac{1}{\delta}}+2\sqrt{v\ln\frac{1}{\delta}}.
        $$
    \end{lemmaAppend}

    \begin{lemmaAppend}
    \label{lem:JCST2022:the_Second_Moment_of_Gradient_smooth_loss_v1}
        With probability at least $1-\Theta(\lceil\ln{T}\rceil)\delta$,
        \begin{align*}
            \sum_{t\in\bar{T}_1\cup \mathcal{J}}\frac{1}
            {(\mathbb{P}[b_t=1])^2}\left\Vert\nabla_{t,i}\right\Vert^2_{\mathcal{H}_i}\cdot\mathbb{I}_{b_t=1}
            \leq& \sum_{t\in\bar{T}_1\cup \mathcal{J}}\frac{\left\Vert\nabla_{t,i}\right\Vert^2_{\mathcal{H}_i}}{\mathbb{P}[b_t=1]}
            +\frac{4G^2_1}{3}\ln\frac{1}{\delta}
            +4\sqrt{G^3_1G_2\hat{L}_{1:T}\ln\frac{1}{\delta}}.
        \end{align*}
    \end{lemmaAppend}
    \begin{proof}
        Define a random variable $X_t$ by
        $$
            X_t=\frac{\Vert\nabla_{t,i}\Vert^2_{\mathcal{H}_i}}
            {(\mathbb{P}[b_t=1])^2}\mathbb{I}_{b_t=1}
            -\frac{\Vert\nabla_{t,i}\Vert^2_{\mathcal{H}_i}}{\mathbb{P}[b_t=1]},\quad
            \vert X_t\vert\leq 2G^2_1.
        $$
        $\{X_{t}\}_{t\in \bar{T}_1}$ forms bounded martingale difference
        w.r.t. $\{b_\tau\}^{t-1}_{\tau=1}$.
        The sum of conditional variances is
        $$
            \Sigma^2\leq \sum_{t\in \bar{T}_1\cup \mathcal{J}}\mathbb{E}[X^2_t]
            \leq 4G^3_1G_2\hat{L}_{1:T}.
        $$
        Using Lemma \ref{lem:JCST2022:improved:Bernstein_ineq_for martingales}
        concludes the proof.
    \end{proof}

    \begin{lemmaAppend}
    \label{lem:JCST2022:proof_small_loss_bound:gradient_variance}
        Let $\Delta_t=\langle f_{t,i}-f,\nabla_{t,i}-\tilde{\nabla}_{t,i}\rangle$.
        With probability at least $1-\Theta(\lceil\ln{T}\rceil)\delta$,
        $$
            \sum_{t\in \bar{T}_1}\Delta_t
            \leq\frac{4UG_1}{3}\ln\frac{1}{\delta}
            +2U\sqrt{2G_2G_1\hat{L}_{1:T}\ln\frac{1}{\delta}}.
        $$
    \end{lemmaAppend}
    \begin{proof}
        The proof is similar to that of Lemma \ref{lem:JCST2022:the_Second_Moment_of_Gradient_smooth_loss_v1}.
    \end{proof}

    \begin{lemmaAppend}
    \label{lemma:JCST2022:property_of_OMD}
        For each $i\in[K]$, let $\psi_i(f)=\frac{1}{2\lambda_i}\Vert f\Vert^2_{\mathcal{H}_i}$.
        Then
        $\mathcal{B}_{\psi_i}(f,g)=\frac{1}{2\lambda_i}\Vert f-g\Vert^2_{\mathcal{H}_i}$ and
        the solutions of \eqref{eq:JCST2022:OMD_first_updating}
        and \eqref{eq:JCST2022:OMD_second_updating}
        are as follows
        \begin{align*}
            f_{t,i}=&f'_{t-1,i}-\lambda_i\hat{\nabla}_{t,i},\\
            f'_{t,i}=&
            \min\left\{1,\frac{U}{\left\Vert f'_{t-1,i}-\lambda_i\nabla_{t,i}\right\Vert_{\mathcal{H}_i}}\right\}
            \cdot\left(f'_{t-1,i}-\lambda_i\nabla_{t,i}\right).
        \end{align*}
        Similarly, we can obtain the solution of \eqref{eq:JCST2022:smooth_loss:mirror_descent}
        $$
            f_{t+1,i}
            =\min\left\{1,\frac{U}{\left\Vert f_{t,i}-\lambda_i\tilde{\nabla}_{t,i}\right\Vert_{\mathcal{H}_i}}\right\}
            \cdot\left(f_{t,i}-\lambda_i\tilde{\nabla}_{t,i}\right).
        $$
        Besides,
        $$
        \langle f_{t,i}-f_{t+1,i},\tilde{\nabla}_{t,i}\rangle
            -\mathcal{B}_{\psi_i}(f_{t+1,i},f_{t,i})
            \leq\frac{\lambda_i}{2}\left\Vert\tilde{\nabla}_{t,i}\right\Vert^2_{\mathcal{H}_i}.
        $$
    \end{lemmaAppend}
    \begin{proof}
        We can solve \eqref{eq:JCST2022:OMD_first_updating}
        and \eqref{eq:JCST2022:OMD_second_updating} by Lagrangian multiplier method.
        Next we prove the last inequality.
        \begin{align*}
            \langle f_{t,i}-f_{t+1,i},\tilde{\nabla}_{t,i}\rangle
            -\mathcal{B}_{\psi_i}(f_{t+1,i},f_{t,i})
            =&\left\langle f_{t,i}-f_{t+1,i},\tilde{\nabla}_{t,i}\right\rangle
            -\frac{\Vert f_{t+1,i}-f_{t,i}\Vert^2_{\mathcal{H}_i}}{2\lambda_i}\\
            =&\frac{\lambda_i}{2}\left\Vert\tilde{\nabla}_{t,i}\right\Vert^2_{\mathcal{H}_i}
           -\frac{1}{2\lambda_i}\left\Vert f_{t+1,i}-f_{t,i} -\lambda_i \hat{\nabla}_i\right\Vert^2_{\mathcal{H}_i}\\
           \leq&\frac{\lambda_i}{2}\left\Vert\tilde{\nabla}_{t,i}\right\Vert^2_{\mathcal{H}_i},
        \end{align*}
        which concludes the proof.
    \end{proof}

%\bibliographystyle{IEEEtran}
%\bibliography{JCST2022appendix}

\end{document}